\documentclass{article}

\PassOptionsToPackage{numbers, compress}{natbib}
\usepackage[preprint]{neurips_2020}

\usepackage[utf8]{inputenc} 
\usepackage[T1]{fontenc}    
\usepackage[hidelinks]{hyperref}
\usepackage{url}
\usepackage{booktabs}       
\usepackage{amsfonts}       
\usepackage{nicefrac}       
\usepackage{microtype}      
\usepackage{amsmath}
\usepackage{amssymb}
\usepackage{mathtools}
\usepackage{amsthm}
\usepackage{algorithm}
\usepackage[noend]{algpseudocode}
\usepackage{subcaption}
\usepackage{graphicx}
\usepackage[dvipsnames]{xcolor}
\usepackage{xfrac}
\usepackage{enumitem}
\usepackage{xspace}
\usepackage{subcaption}
\usepackage{booktabs}
\usepackage{setspace}
\usepackage{nccmath}
\usepackage{bm}
\usepackage{wrapfig}

\newcommand{\E}[1]{\mathrm{E}\hspace{-1.0px}\left[#1\right]}
\newcommand{\Edist}[2]{\mathrm{E}_{\hspace{.6px}#1}\hspace{-2px}\left[#2\right]}
\newcommand{\entropy}[2]{\mathrm{H}_{#1}\hspace{-1.5px}\left(#2\right)}
\newcommand{\Prob}[1]{\mathrm{P}\hspace{-1.5px}\left(#1\right)}

\newcommand{\defeq}{\coloneqq}

\renewcommand{\cal}[1]{\mathcal{#1}}
\newcommand{\cV}{\cal{V}}
\newcommand{\cE}{\cal{E}}
\newcommand{\cX}{\cal{X}}
\newcommand{\cY}{\cal{Y}}

\newcommand{\cL}{\cal{L}}
\newcommand{\cS}{\cal{S}}

\newcommand{\cA}{\cal{A}}

\newcommand{\bb}[1]{\mathbb{#1}}
 
\newcommand{\bbD}{\bb{D}}  
\newcommand{\bbT}{\bb{T}}  
\newcommand{\bbG}{\bb{G}}
\newcommand{\bbB}{\bb{B}}

\newcommand{\myfrac}[2]{\sfrac{#1\hspace{-1.0pt}}{#2}}

\newcommand{\net}{g}
\newcommand{\comnet}[1]{N(#1)}

\newtheorem{theorem}{Theorem}[section]

\newtheorem{corollary}{Corollary}[section]
\newtheorem{lemma}{Lemma}[section]
\newtheorem{claim}{Claim}[section]
\theoremstyle{definition}
\newtheorem{definition}{Definition}[section]

\renewcommand{\paragraph}[1]{\vspace{1.0mm}\noindent\textbf{#1}}
\setlist[itemize]{itemsep=0.2em, labelsep*=0.8em} 

\newlength{\examplelength}
\title{How hard is to distinguish graphs \\with graph neural networks?}

\author{
  Andreas Loukas \\
  \'{E}cole Polytechnique F\'{e}d\'{e}rale Lausanne\\
  \texttt{andreas.loukas@epfl.ch} \\
}

\begin{document}

\maketitle

\begin{abstract}
A hallmark of graph neural networks is their ability to distinguish the isomorphism class of their inputs. This study derives hardness results for the classification variant of graph isomorphism in the message-passing model (MPNN). MPNN encompasses the majority of graph neural networks used today and is universal when nodes are given unique features. The analysis relies on the introduced measure of \textit{communication capacity}. Capacity measures how much information the nodes of a network can exchange during the forward pass and depends on the depth, message-size, global state, and width of the architecture. It is shown that the capacity of MPNN needs to grow linearly with the number of nodes so that a network can distinguish trees and quadratically for general connected graphs. The derived bounds concern both worst- and average-case behavior and apply to networks with/without unique features and adaptive architecture---they are also up to two orders of magnitude tighter than those given by simpler arguments. An empirical study involving 12 graph classification tasks and 420 networks reveals strong alignment between actual performance and theoretical predictions. 
\end{abstract}

\section{Introduction}
\label{sec:intro}

A fundamental goal in the analysis of graph neural networks is to determine under what conditions current networks can (or perhaps cannot) distinguish between different graphs~\citep{maron2019universality,keriven2019universal,xu2018powerful,morris2019weisfeiler,DBLP:journals/corr/abs-1905-12560,chen2020can}. 
The most intensely studied model in the literature has been that of message-passing neural networks (MPNN). Since its inception by~\citet{scarselli2008graph}, MPNN has been extended to include edge~\citep{gilmer2017neural} and global features~\citep{battaglia2018relational}. The model also encompasses many of the popular graph neural network architectures used today~\citep{kipf2016semi,xu2018powerful,hamilton2017inductive,li2015gated,duvenaud2015convolutional,battaglia2016interaction,kearnes2016molecular,simonovsky2017dynamic}.

Roughly two types of analyses of MPNN may be distinguished.
The first bound the expressive power of \textit{anonymous} networks, i.e., those in which nodes do not have any access to node features (also known as labels or attributes) and that are permutation equivariant by design. \citet{xu2018powerful} and \citet{morris2019weisfeiler} established the equivalence of anonymous MPNN to the 1st-order Weisfeiler-Lehman (1-WL) graph isomorphism test. 
A consequence of this connection is that anonymous MPNN cannot distinguish between regular graphs with the same number of nodes, but can recognize trees as long as the MPNN depth exceeds the tree diameter. 
Other notable findings include the observation that MPNN cannot count simple subgraphs~\citep{chen2020can}, as well as the analysis of the power of particular architectures to compute graph properties~\citep{dehmamy2019understanding,garg2020generalization} and to distinguish graphons~\citep{magner2020power}---see also~\citep{geerts2020let,Sato2020ASO,DBLP:journals/corr/abs-1905-12560}. 

The aforementioned insights can be pessimistic in the \textit{non-anonymous} case, where permutation equivariance is either learned from data~\citep{murphy2019relational,Loukas2020a} or obtained by design~\citep{vignac2020building}.  
With node features acting as identifiers, MPNN were shown to become universal \textit{in the limit}~\citep{Loukas2020a}, which implies that they can solve the graph isomorphism testing problem if their size is allowed to depend exponentially on the number of nodes~\cite{DBLP:journals/corr/abs-1905-12560}. The node features, for instance, may correspond to a one-hot encoding~\citep{kipf2016semi,berg2017graph,murphy2019relational} or a random coloring~\citep{dasoulas2019coloring,sato2020random}.

At the same time, universality statements carry little insight about the power of practical networks, as they only account for behaviors that occur in the limit. 
Along those lines,  
recent work provided evidence that the power of MPNN grows as a function of depth and width for certain graph problems~\citep{Loukas2020a}, showing that (both anonymous and non-anonymous) MPNN cannot solve many tasks when the product of their depth and width does not exceed a polynomial of the number of nodes.  
Nevertheless, it remains an open question whether similar results hold also for problems relating to the capacity of MPNN to distinguish graphs.
Even further, it is unclear whether depth and width needs to grow with the number of nodes solely in the worst-case (as proven in~\citep{Loukas2020a}) or with certain probability over the input distribution.

\begin{figure}[t]
\begin{subfigure}[b]{0.32\columnwidth}
\includegraphics[width=1\columnwidth,trim=0mm 0mm 0mm 0mm]{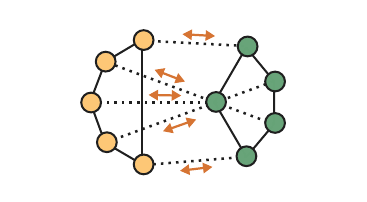}
\vspace{-0.4cm}\caption{communication capacity}
\label{fig:capacity}
\end{subfigure}
\hfill
\begin{subfigure}[b]{0.32\columnwidth}
\includegraphics[width=1.0\columnwidth,trim=0mm 0 0mm 0mm, clip]{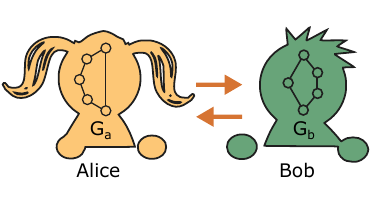}
\caption{communication complexity}
\label{fig:complexity}
\end{subfigure}
\hfill
\begin{subfigure}[b]{0.32\columnwidth}
\includegraphics[width=1.0\columnwidth,trim=0mm 0 0mm 0mm, clip]{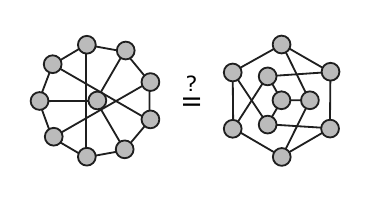}
\caption{graph isomorphism testing}
\label{fig:gi}
\end{subfigure}
\caption{(a) In MPNN, nodes exchange information by sending and receiving messages along edges. Communication capacity is the maximal amount of information that can be sent across two subgraphs (depicted in orange and green) (b) Communication complexity is the minimal amount of information 
needed so that two parties jointly compute a function $f$. (c) To determine whether graphs $G$ and $G'$ are isomorphic one may use an MPNN $\net$ to test whether $\net(G) = \net(G')$. \vspace{-3mm}}
\label{fig:main}
\end{figure}

\subsection{Communication capacity and its consequences to distinguishing graphs}

Aiming to study the power of MPNN of practical size to distinguish graphs, this paper defines and characterizes \textit{communication capacity}, a measure of the amount of information that the nodes of the network can exchange during the forward pass (see Figure~\ref{fig:capacity}). In Section~\ref{sec:mpnn} it is shown that the capacity of MPNN depends on the network's depth, width, and message-size, as well as on the cut-structure of the input graph. Communication capacity is an effective generalization of the previously considered product between depth and width~\citep{Loukas2020a}, being able to consolidate more involved properties, as well as to characterize MPNN with global state~\citep{gilmer2017neural,battaglia2018relational,ishiguro2019graph} and adaptive architecture~\citep{graves2016adaptive,spinelli2020adaptive,velivckovic2019neural,corso2020principal}.

The paper then delves into the \textit{communication complexity} of determining the {graph isomorphism} class. The theory of communication complexity compliments the definition of communication capacity as it provides a convenient mathematical framework to study how much information needs to be exchanged by parties that jointly compute a function~\citep{rao_behudayoff_2020} (see Figure~\ref{fig:complexity}). 
In this setting, Section~\ref{sec:communication_complexity} derives    
hardness results for determining the isomorphism class of connected graphs and trees. It is shown that the communication capacity of any MPNN needs to grow at least {linearly} with the number of nodes so that the network can learn to distinguish {trees}, and {quadratically} to distinguish between {connected graphs}. 
The analysis stands out from previous relevant works that have studied subcases of isomorphism, such as subgraph freeness~\citep{even2017three,gonen2018lower} or those focused on anonymous networks~\citep{xu2018powerful,morris2019weisfeiler,DBLP:journals/corr/abs-1905-12560,chen2020can,dehmamy2019understanding,magner2020power}. In fact, the derived hardness results apply to \textit{both} anonymous and non-anonymous MPNN and can be up to two orders of magnitude tighter than what can be deduced from simpler arguments. 
In addition, the proposed lower bounds rely on a new technique which renders them applicable not only to worst-case instances~\citep{Loukas2020a}, but in expectation over the input distribution.  

An empirical study reveals strong qualitative and quantitative agreement between the MPNN test accuracy and theoretical predictions. In the 12 graph isomorphism tasks considered, the performance of the 420 graph neural networks trained was found to depend strongly on their communication capacity. In addition, the proposed theory could consistently predict which networks would exhibit poor classification accuracy as a function of their capacity and the type of task in question.

\section{Communication complexity for message-passing networks}
\label{sec:mpnn}

Suppose that a learner is given a graph $G = (\cV, \cE, a)$ sampled from a distribution $\bbD$ that is supported over a finite universe of graphs $\cX$. Throughout this paper, $\cV$ will denote the set of nodes of cardinality $n$, $\cE$ the set of edges, and $a$ encodes any node and edge features of interest.
With $G$ as input, the learner needs to predict the output of function $f: \cX \to \cY$. This work focuses on graph classification, in which case $f$ assigns a class $y \in \cY$ (i.e., its isomorphism class) to each graph in the universe. 

\subsection{Message-passing neural networks (MPNN)}

In MPNN, the node representation $x_i^{(\ell)}$ of every node $v_i \in \cV$ is initialized to be equal to the node's attributes $x_{i}^{(0)} = a_i$ and is progressively updated by exchanging messages:
$$
            x_i^{(\ell)} = \textsc{Update}_{\ell} \big(x_i^{(\ell-1)}, \ \big\{ \text{msg}_{ij}^{(\ell)} \ : \ e_{ij} \in \cE \big\} \big)
        \quad \text{for} \quad \ell = 1, \ldots, d,
$$
where each message 
$$
\text{msg}_{ij}^{(\ell)} = \textsc{Message}_{\ell}\big(x_j^{(\ell-1)},\, a_j,\,  a_{ij} \big)
$$ 
contains some information that is sent to from node $v_j$ to $v_i$.

Every neuron in a network utilizes some finite alphabet $\cS$ containing $s = |\cS|$ symbols to encode its state. For this reason, $x_i^{(\ell)}$ and $\text{msg}_{ij}^{(\ell)}$ are selected from $\cal{S}^{w_{\ell}}$  and $\cal{S}^{m_{\ell}}$, where $w_\ell$ and $m_{\ell}$ are the width (i.e., number of channels) and the message-size of the $\ell$-th layer. For instance, to represent whether a neuron is activated one uses binary symbols, whereas a practical implementation could use as symbols the set of numbers represented in floating-point arithmetic. 
$\textsc{Message}_{\ell}$ and $\textsc{Update}_{\ell}$ are layer-dependent functions whose parameters are selected based on some optimization procedure. It is common to parametrize these functions by feed-forward neural networks~\citep{scarselli2008graph,li2015gated,battaglia2018relational}. The rational is that, by the universal approximation theorem and its variants~\citep{cybenko1989approximation,hornik1989multilayer,NIPS2017_7203}, these networks can approximate any smooth function that maps vectors onto vectors. 
If the network's output is required to be independent of the number of nodes, the output is recovered from the representations of the last layer by means of a readout function: 
$
\net(G) = \textsc{ReadOut}\big( \big\{ x_i^{(d)} : v_i \in \cV\big\} \big).
$
For simplicity, it is here assumed that no graph pooling is employed~\citep{ying2018hierarchical,bianchi2019mincut}, though the results may also be easily extended to account for coarsening~\citep{chen2011algebraic, safro2015advanced,loukas2018spectrally,loukas2019graph,jin2020graph}.
 
\textbf{Global state.} In the description above, all message exchange needs to occur along graph edges. However, one may also easily incorporate a \textit{global state} (or external memory) to the model above by instantiating a special node $v_0$ and extending the edge set to contain edges from every other node to it. Global state is useful for incorporating graph features to the decision making~\citep{battaglia2018relational} and there is some evidence that it can facilitate logical reasoning~\citep{barcelo2019logical}. Here, I will suppose that $x_0^{(\ell)}$ belongs to set $\cS^{\gamma_{\ell}}$.

\textbf{Adaptive MPNN.} The forward-pass of an MPNN concludes after $d$ layers. However, the depth of a network may be adaptive~\citep{graves2016adaptive,spinelli2020adaptive,velivckovic2019neural,corso2020principal}. In particular, $d$ may depend on the size and connectivity of the input graph or any adaptive computation time heuristic~\cite{graves2016adaptive,spinelli2020adaptive} based on, for example, the convergence of the node representation~\citep{velivckovic2019neural,corso2020principal}. In the same spirit, in the following it will supposed that \textit{all} hyper-parameters of an MPNN, such as its depth, width, message-size, and global state size, can be adaptively decided based on the input graph $G$.

\subsection{Communication capacity}  

An MPNN $\net$ can be thought of as a communication network $\comnet{G, \net}$, having processors as nodes and with connectivity determined by the input graph $G$. $\comnet{G, \net}$ operates in $\ell=1, \ldots, d$ synchronous communication rounds and $m_{\ell}$ symbols are transmitted in round $\ell$ from each processor $v_i$ to each one its neighbors $v_j$ such that $e_{ij} \in \cE$. Further, the processors have limited and round-dependent memory: in round $\ell$ the processors corresponding to nodes $\cV$ can store $w_\ell$ symbols, whereas the external memory processor $v_0$ can store $\gamma_{\ell}$ symbols.

The communication complexity of a message-passing neural network corresponds to the maximum amount of information that can be sent in $\comnet{G, \net}$ between disjoint sets of nodes:
\begin{definition}[Communication capacity]
Let $\net$ be an MPNN and fix a graph $G=(\cV, \cE)$.
For any two disjoint sets $\cV_a, \cV_b \subset \cV$, the communication capacity $c_{\net}$ of $\net$ is the maximum number of symbols that $\comnet{G, \net}$ can transmit from $\cV_a$ to $\cV_b$ and from $\cV_b$ to $\cV_a$. 
\label{def:capacity}
\end{definition}

To understand Definition~\ref{def:capacity}, imagine that the node-disjoint subgraphs $G_a = (\cV_a, \cE_a)$ and $G_b = (\cV_b, \cE_b)$ of $G$ are controlled by two parties: Alice and Bob (see Figure~\ref{fig:main}). In practice, Alice and Bob correspond to two sub-networks of $\net$. By construction, when Alice needs to send information to Bob, she does so by sending information across some paths that cross between $\cV_a$ and $\cV_b$. Bob does the same. From this elementary observation, it can be deduced that the number of symbols that can be sent during the forward pass is bounded by the cut between the two subgraphs:
\begin{lemma}
Let $\net$ be an MPNN of $d$ layers, where each has width $w_\ell$ (i.e., number of channels), exchanges messages of size $m_\ell$, and maintains a global state of size $\gamma_{\ell}$. For any disjoint partitioning of $\cV$ into $\cV_a$ and $\cV_b$, 
the communication complexity of $\net$ is at most 
$$
	 c_{\net} \leq \text{cut}(\cV_a, \cV_b) \sum_{\ell = 1}^d \min\{ m_\ell, w_{\ell}\} + \sum_{\ell = 1}^d \gamma_{\ell},
$$
with $\text{cut}(\cV_a, \cV_b)$ being the size of the smallest cut that separates $\cV_a$ and $\cV_b$ in $G$.
\label{lemma:maxflow}
\end{lemma}

Whenever the MPNN involves sending for each $e_{ij} \in \cE$ two messages, i.e., one from $v_i$ to $v_j$ and one from $v_j$ to $v_i$, every edge should be counted twice in the calculation of $\text{cut}(\cV_a, \cV_b)$.

It is also interesting to remark that $c_{\net}$ may be a random quantity. In particular, when $G$ is sampled from a distribution $\bbD$, the capacity of an \textit{adaptive} MPNN, i.e., a network whose hyper-parameters change as a function the input, may vary as well. For this reason, the analysis will also consider the expected communication capacity $c_\net(\bbD)$ of $\net$ w.r.t. $\bbD$.

\subsection{Communication complexity}
\label{subsec:communication_complexity_definitions}

Let us momentarily diverge from graphs and suppose that Alice and Bob wish to jointly compute a function $f: \cX_a \times \cX_b \to \cY$ that depends on both their inputs. Alice's input is an element $x_a\in \cX_a$ and Bob sees an element $x_b \in \cX_b$. Later on, $x_a$ and $x_b$ will correspond to $G_a$ and $G_b$, respectively, whereas $y \in \cY$ will be the classification output (see Figure~\ref{fig:complexity}). 

To compute $f(x_a,x_b)$, the two parties need to exchange information based on some communication \textit{protocol} $\pi$. Concretely, $\pi$ determines for each input ($x_a, x_b$) the sequence $\pi(x_a,x_b) = ((\textsc{id}_1,s_1), (\textsc{id}_2,s_2), \ldots)$ of symbols that are exchanged, with each symbol $s_i \in \cS$ being paired with the id of its sender (Alice or Bob)---for a more detailed description, see Appendix~\ref{app:com_complexity}. 
The number of symbols exchanged by $\pi$ to successfully compute $f(x_a,x_b)$ are denoted by $\|\pi(x_a,x_b)\|_{m}$, with subscript $m \in \{\textit{one},\textit{both}\}$ indicating whether ``successful computation'' entails one or both parties figuring out $f(x_a,x_b)$ at the end of the exchange. 

\paragraph{Worst-case complexity.}
The focus of classical theory is on the worst-case input. The \textit{communication complexity}~\citep{rao_behudayoff_2020} of $f$ is defined as
\begin{align}
    c_f^\textit{\hspace{1px}m} \defeq \min_{\pi} \max_{(x_a,x_b) \in \cX_a \times \cX_b} \|\pi(x_a,x_b)\|_m
\end{align}
and corresponds to the minimum worst-case length of any protocol that computes $f$.

\paragraph{Expected complexity.} In machine learning, one usually cares about the expected behavior of a learner when its input is sampled from a distribution. Concretely, let $(X_a,X_b)$ be random variables sampled from a distribution $\bbD$ with domain $\cX_a\times \cX_b$. The expected length of a protocol $\pi$ is
\begin{align}
    \Edist{\bbD}{c_f^m (\pi)} \defeq \sum_{(x_a,x_b) \in \cX_a \times \cX_b} \|\pi(x_a,x_b)\|_m  \cdot \Prob{X_a=x_a, X_b=x_b},
\end{align}
where now the protocol length $\|\pi(x_a,x_b)\|_m $ is weighted according to the probability of each input. 
With this in place, I define the \textit{expected communication complexity} of $f$ as 
\begin{align}
    c_f^m ({\bbD}) \defeq \min_{\pi} \Edist{\bbD}{c_f^m (\pi)},
\end{align}
corresponding to the minimum expected length of any protocol that computes $f$.

For an overview of the classical theory of communication complexity pertaining to the worst-case and an analysis of the newly-defined expected complexity, the reader may refer to Appendix~\ref{app:com_complexity}.

To use communication complexity for learning problems $f: \cX \to \cY$ from a graph universe $\cX$ to a set of classes $\cY$ one needs to decompose (a subset of) $\cX$ as $\cX_a \times \cX_b$. As it will be seen in the following sections, the decomposition can be achieved by finding a disjoint partitioning of every graph $G\in \cX$ into subgraphs $G_a \in \cX_a$ and $G_b \in \cX_b$, held by Alice and Bob, respectively. Then, in the worst case, $c_f^\textit{\hspace{1px}m}$ symbols need to be exchanged so that one (\textit{m}=\textit{one}) or both (\textit{m}=\textit{both}) parties can correctly classify $G$ into class $y = f(G)$. Moreover, if $G$ is sampled from some distribution $\bbD$, then the two parties need to exchange at least $c_f^\textit{m}({\bbD})$ symbols in expectation. 
Together with Lemma~\ref{lemma:maxflow}, the aforementioned bounds can be used to characterize what an MPNN cannot achieve as a function of its worst-case and expected capacity.

\section{Hardness results for determining the isomorphism class}
\label{sec:communication_complexity}

This section derives necessary conditions for the communication capacity of a network that determines the graph isomorphism class of its inputs. 
This entails finding a mapping $f_\text{isom}: \cX \to \cY$ from a universe of labeled graphs to their corresponding isomorphism classes. Crucially, though the nodes of graph $G$ are assigned some predefined order (which constitutes their label in graph-theory nomenclature), the class $f_\text{isom}(G)$ should be invariant to this ordering.

As it will be shown, MPNN of sub-quadratic and sub-linear capacity cannot compute the isomorphism class of connected graphs and trees, respectively: 

\begin{theorem}
Let $\net$ be a MPNN using either a majority-voting or a consensus based readout (defined in Section~\ref{subsec:consequences_mpnn}). Denote by $c_\net$ its communication capacity.  
\begin{enumerate}
\item To compute $f_{\text{isom}}$ for every graph and tree of $n$ nodes, it must be that
$
c_\net = \Omega\left(n^2\right)$ 
and
$c_\net = \Omega\left(n\right)$, respectively.
 
\item If each graph is sampled from $\bbB_{n/2,p}$ (defined in Theorem~\ref{theorem:graph_isomorphism_random}) to compute $f_{\text{isom}}$ in expectation it must be that
$
c_\net(\bbD) = \Omega\left(n^2\right). 
$
Further, if each graph is a tree sampled from $\bbT_{n/2}$ (defined in Theorem~\ref{theorem:tree_isomorphism}) to compute $f_{\text{isom}}$ in expectation it must be that
$
c_\net(\bbD) = \Omega\left(n\right). 
$
\end{enumerate}
\label{theorem:main}
\end{theorem}

For general graphs, these results are one or two orders of magnitude tighter than arguments that compare the receptive field of a neural network with the graph diameter. Specifically, connected graphs have diameter at most $n$ and thus a diameter analysis yields $d = \Omega(n)$ without a global state and $d = \Omega(1)$ with one (as any two nodes are connected by a path passing through $v_0$).

The tree distribution was chosen purposefully to demonstrate that the bounds are also relevant for the anonymous case, when MPNN can also be analyzed by equivalence to the 1-WL test~\citep{xu2018powerful,morris2019weisfeiler}. For trees, the 1-WL test requires $n$ iterations because there exists a tree of diameter $n$. However, since MPNN is equivalent to 1-WL only when the former is built using injective aggregation functions (i.e., of unbounded width~\citep{xu2018powerful,morris2019weisfeiler,DBLP:journals/corr/abs-1905-13192}), the equivalence does not imply a relevant lower bound on the width/message-size/global-state-size of MPNN. Further, the communication complexity analysis introduced here yields tighter results in expectation: it asserts that one needs $\Omega(n)$ capacity on average, even though the average tree in $\bbT_{n/2}$ has $O(\sqrt{n})$ diameter (and thus 1-WL would require $d=\Omega(\sqrt{n})$ in expectation).

\textbf{Graph isomorphism testing.} There is also a close relation between $f_\text{isom}$ and the graph isomorphism testing problem (see Figure~\ref{fig:gi}). Specifically, methods for isomorphism testing~\citep{xu2018powerful,morris2019weisfeiler,chen2020can} that compare graphs $G$ and $G'$ by means of some invariant representation or embedding 
$$
\net(G) = \net(G') \quad  \text{if and only if} \quad f_\text{isom}(G) = f_\text{isom}(G')
$$
can be expressed as $g = q \circ f_{\text{isom}}$ for some injective function $q$. Since $q$ does not involve any exchange of information, {the communication complexity of such testing methods is the same as that of $f_\text{isom}$}. The proposed hardness results thus still hold.

The rest of this section is devoted to proving Theorem~\ref{theorem:main}. The analysis consists of two parts: the communication complexity of distinguishing graphs and trees is derived in Section~\ref{subsec:com_comp_analysis}, and the implications of these results to MPNN are discussed in Section~\ref{subsec:consequences_mpnn}. 

\subsection{Communication complexity analysis}
\label{subsec:com_comp_analysis}

\begin{figure}[t]
    \centering
    \includegraphics[width=0.52\columnwidth]{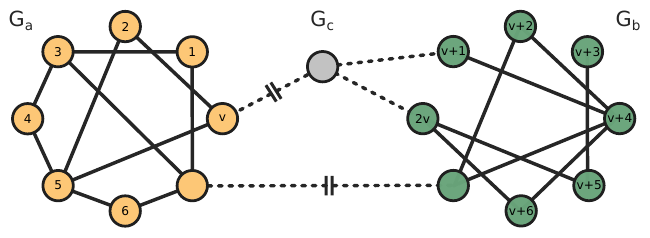}
    \caption{A visual depiction of a graph $G = (\cV, \cE)$ chosen from $\cX$. $G_a$ (in yellow) and $G_b$ (in green) are chosen from families $\cX_a$ and $\cX_b$ of graphs with $v$ nodes. The edges of $G_c$ (dashed lines) may connect to any node but should induce a $(\cV_a, \cV_b)$-cut of at most $\tau$. }
    \label{fig:my_label}
\end{figure}

Rather than focusing directly on the universe of all graphs and trees, respectively, it will be convenient to analyze a strictly smaller universe $\cX$ containing easily partitioned graphs. As it will be seen, we can utilize such a restriction without significant loss of generality, because the derived worst-case impossibility results also apply to any universe that is a strict superset of $\cX$. 

Concretely, $\cX$ will consist of all labeled graphs $G = (\cV, \cE)$ on $n$ nodes admitting to the following $(\cX_a, \cX_b, \tau)$ decomposition: 
\begin{enumerate}
    \item[(a)] Subgraph $G_a = (\cV_a, \cE_a)$ induced by labels $\cV_a = (1, 2, \cdots, v)$  belongs to $\cX_a$.
    \item[(b)] Subgraph $G_b = (\cV_b, \cE_b)$ induced by labels $\cV_b = (v+1, v+2, \cdots, 2v)$  belongs to $\cX_b$.
    \item[(c)] Subgraph $G_c = (\cV, \cE \setminus (\cE_a \cup \cE_b))$ yields $\textit{cut}(\cV_a, \cV_b) \leq \tau$. 
\end{enumerate}

An example $(\cX_a, \cX_b, \tau)$ decomposable graph is depicted in Figure~\ref{fig:my_label}. This decomposition is fairly general: the main restriction placed is that the cut between $\cV_a$ and $\cV_b$ is bounded by $\tau$. Families $\cX_a$ and $\cX_b$ can be chosen to contain relevant families of graphs (e.g., all connected graphs or all trees), whereas $G_c$ may be selected arbitrarily. 
To derive lower bounds, it will be imagined that $G_a$ and $G_b$ are known by Alice and Bob, respectively, while both know $G_c$.
The goal of the two parties is to determine $f_\text{isom}(G)=f_\text{isom}(G_a, G_b, G_c)$ by exchanging as little information as possible.  

Two main results will be proven: Section~\ref{subsec:graph_isomorphism} will show that, when $\cX_a$ and $\cX_b$ contain all labeled connected graphs on $v$ nodes, the worst-case and expected communication complexity are both $\Theta(v^2)$. Moreover, in Section~\ref{subsec:tree_isomorphism} it is proven that, when $\cX_a$ and $\cX_b$ contain only trees, the two complexities are $\Theta(v)$. 
%

\subsubsection{Distinguishing connected graphs}
\label{subsec:graph_isomorphism}

When $\cX_a$ and $\cX_b$ contain all connected graphs on $v$ nodes, Alice and Bob should exchange $\Theta(v^2)$ symbols in the worst case: 
\begin{theorem}[Worst-case complexity]
When $\cX_a$  and $\cX_b$ each contain the set of all connected graphs on $v$ nodes, the worst-case communication complexity of $f_\text{isom}$ is at least 
$$
O(v^2) = c_{f_\text{isom}}^\textit{both} 
	\geq \frac{v^2}{\log_2{s}} - 2v\log_s{\left(\frac{v\sqrt{2}}{e}\right)} - \log_s{\left(2 ve^2\right)} + o(1) = \beta = \Omega\left(v^2\right)
$$
and $O(v^2) = c_{f_\text{isom}}^{\textit{one}} \geq \frac{\beta - (\log_2{s})^{-1}}{2} = \Omega\left(v^2\right)$.
\label{theorem:graph_isomorphism_wc}
\end{theorem}

A similar bound holds also in the \textit{random graph model} $\bbG_{v,p}$. In $\bbG_{v,p}$, every graph with $v$ nodes and $k$ edges is sampled with probability  
$$
    \Prob{G \sim \bbG_{v,p}} = p^{k} (1-p)^{\binom{v}{2} - k}.
$$
Effectively, this means the probability of choosing each graph depends only on the number of edges it contains. Moreover, for $p=0.5$ each graph is sampled uniformly at random from the set of all possible graphs.  
The following theorem bounds the expected communication complexity when the subgraphs known to Alice and Bob are sampled from $\bbG_{v,p}$:

\begin{theorem}[Expected complexity]
Let $G_a$ and $G_b$ be sampled independently from $\bbG_{v,p}$, with $\log{v}/v < p < 1-s^{\Omega(1)}$ and $\text{cut}(\cV_a, \cV \setminus \cV_a) = \text{cut}(\cV_b, \cV \setminus \cV_b) = 1$. Denote by $\bbB_{v,p}$ the resulting distribution. With high probability, 
$$
    O(v^2) = c_{f_\text{isom}}^\textit{both}(\bbB_{v,p}) 
    \geq v^2 \, \entropy{s}{p} - v \left( 2 \log_s{\left(\frac{v}{e} \right)} + \entropy{s}{p}\right) - \log_s{\left(2ve^2\right) } = \beta 
 = \Omega(v^2)
$$
and
$$    
	O(v^2) = c_{f_\text{isom}}^{\textit{one}}(\bbB_{v,p}) \geq \frac{ \beta}{2} - \frac{v^2 - v(1-\entropy{2}{p}) + 1}{2 \log_2{s}} = \Omega(v^2),
$$
where $\entropy{s}{p} = -(1-p) \log_s{(1-p)} - p \log_s{p}$ is the binary entropy function (base $s$). 
\label{theorem:graph_isomorphism_random}
\end{theorem}

The expected complexity, therefore, grows asymptotically with $\Theta(v^2)$ and is maximized when every graph in the universe is sampled with equal probability, i.e., for $p = 0.5$. Interestingly, in this setting, the bounds of Theorems~\ref{theorem:graph_isomorphism_wc} and Theorem~\ref{theorem:graph_isomorphism_random} match.
This implies that, unless there is some strong isomorphism class imbalance in the dataset, the communication complexity lower bound posed by Theorem~\ref{theorem:graph_isomorphism_wc} does not only concern rare worst-case inputs, but should be met on average.

In the theorem it is asserted that $\log{v}/v < p < 1-s^{\Omega(1)}$. The aforementioned lower bound suffices to guarantee that every $G \sim \bbB_{v,p}$ will be connected with high probability, whereas the upper bound is needed to ensure that $\entropy{s}{p} = \Omega(1)$.

\subsubsection{Distinguishing trees}
\label{subsec:tree_isomorphism}

Distinguishing trees (connected acyclic undirected graphs) is significantly easier: 

\begin{theorem}
Suppose that $G_a$ and $G_b$ are sampled independently from the set of all trees on $v$ nodes. Denote by $\bb{T}_{v}$ the resulting distribution. The communication complexity of $f_\text{isom}$ is at least
$$ 
O(v) = c_{f_\text{isom}}^\textit{both} \geq c_{f_\text{isom}}^\textit{both}(\bbT_v) \gtrsim  2v \log_s{\alpha} - 5 \log_s{v} + \log_s{7} = \beta = \Omega(v) 
$$
and 
$O(v) = c_{f_\text{isom}}^\textit{one} \geq c_{f_\text{isom}}^\textit{one}(\bbT_v) \gtrsim  \frac{\beta + \log_{s}2}{2} = \Omega(v),
$  
where $\alpha \approx 2.9557652$ and $f(n) \gtrsim g(n)$ means $f(n) \geq g(n)$ as $n$ grows.
\label{theorem:tree_isomorphism}
\end{theorem}

Akin to the general case, the expected and worst-case complexities match when every tree is sampled with equal probability. Since a distribution over trees cannot be meaningfully parametrized based a connection probability $p$ (trees always have the same number of edges), by default in $\bbT_{v}$ every $G \in \cX$ is sampled with equal probability.

\subsection{Consequences for message-passing neural networks}
\label{subsec:consequences_mpnn}

Two types of networks are distinguished depending on how the readout function operates: 
\begin{enumerate}
\vspace{-1mm}
\item \textsc{ReadOut} performs \textit{majority-voting}. Specifically, for $\net$ to compute $f_{\textit{isom}}(G)$ there should exist a function $r : \cS^{w_d} \to \cY$ and a set of nodes $\cal{M}_G \subseteq \cV$ possibly dependent on $G$ and of cardinality at least $|\cal{M}_G| \geq \mu = O(1)$, such that $ r(x_{i}^{(d)}) = f_{\textit{isom}}(G)$ for every $v_i \in \cal{M}_G$.

\item \textsc{ReadOut} performs \textit{consensus}. This is akin to a majority-voting, with the distinction that $\cal{M}_G$ should contain at least $|\cal{M}_G| \geq n - \mu = \Omega(n)$ nodes.
\end{enumerate}
The implications of a communication complexity bound to MPNN capacity are as follows:  
\begin{lemma}
Let $\bbD$ be a distribution over graphs that is supported on a universe $\cX$ admitting to a $(\cX_a, \cX_b, \tau)$ decomposition.
Further, suppose that $\net$ is an MPNN whose communication capacity is always bounded from above by $c_{\net}$ and is at most $c_{\net}(\bbD)$ in expectation. The following hold:
\begin{enumerate}

\item There exists some $G \in \cX$ for which computing $f_\text{isom}(G)$ necessitates $c_{\net} \geq c_{f_\text{isom}}^{m}$. In addition, for every $\cX' \supset \cX$ network $\net$ cannot compute $f_\text{isom}(G)$ for some $G \in \cX'$. 

\item In expectation, computing $f_\text{isom}$ necessitates $ c_{\net}(\bbD) \geq c_{f_\text{isom}}^{m}(\bbD)$. Moreover, if $ c_{\net} < \delta \, c_{f_\text{isom}}^{m}(\bbD) $ for some $\delta \in [0,1]$, then $\net$ cannot compute $f_\text{isom}(G)$ with probability at least ${(1-\delta)}/{((\beta_m/c_{f_\text{isom}}^m(\bbG)) - \delta)}$.
\end{enumerate}
Above, with majority-voting one should set $\textit{m} = \textit{one}$ and $v > (n - \mu)/2$, whereas with consensus $\textit{m} = \textit{both}$ and $v > \mu$. Further, $\beta_m$ is the worst-case length of a protocol with optimal expected length.
\label{lemma:mpnn-complexity}
\end{lemma}

With Lemma~\ref{lemma:mpnn-complexity} in place, the proof of Theorem~\ref{theorem:main} follows from Theorems~\ref{theorem:graph_isomorphism_wc},~\ref{theorem:graph_isomorphism_random} and~\ref{theorem:tree_isomorphism} by setting $v=n/2$.

\section{Empirical results}
\label{sec:empirical}

This section tests the developed theory on 12 graph and tree isomorphism classification tasks of varying difficulty.
In the 420 neural networks tested, the bounds are found to consistently predict when each network can solve a given task as a function of its capacity.

\subsection{Experimental setting}
MPNN of different capacities were tasked with learning the mapping between a universe of graphs their corresponding isomorphism classes. 

\textit{Datasets.} 
A total of 12 universes were constructed following the theory: $\cX_{\text{graph}}^n$ for $n = (6,8,10,12)$ and $\cX_{\text{tree}}^n$ for $n = (8,10,\ldots,22)$. 
Each $\cX_{\text{graph}}^n$ was built in two steps: First, \texttt{geng}~\citep{McKay201494} was used to populate $\cX_a$ and $\cX_b$ with all possible connected graphs on $v=\myfrac{n}{2}$ nodes. Then, each $G \in \cX_{\text{graph}}^n$ was generated by selecting $G_a$ and $G_b$ from $\cal{X}_a$ and $\cX_b$ and connecting them with an edge, such that $\tau =1$. The labels added to the nodes of $G$ were the one-hot encoding of a random permutation of $(1,\ldots,v)$ and $({v+1}, \ldots, n)$.    
The construction of $\cX_{\text{tree}}^n$ differed only in that $\cX_a$ and $\cX_b$ contained all trees on $v=\myfrac{n}{2}$ nodes. 
Then, the 12 datasets were built by sampling graphs from each respective universe. These were split into a training, a validation, and a test set (covering 90\%, 5\%, and 5\% of the dataset, respectively). Additional details are provided in Appendix~\ref{app:experimental_details}.

\textit{Architecture and training.} 
The networks combined multiple GIN0~\citep{xu2018powerful} layers with batch normalization and a simple sum readout. Their depth and width varied in $d \in (2,3,4,5,6,7,8)$ and $w \in (1,2,4,8,16)$, respectively, the message-size was set equal to $w$, and no global state was used. Each network was trained using Adam with a decaying learning rate. Early stopping was employed when the validation accuracy reached 100\%.

\begin{figure}[t!]
\hspace{-0.05cm}
\begin{subfigure}[b]{0.32\columnwidth}
\includegraphics[width=1.04\columnwidth]{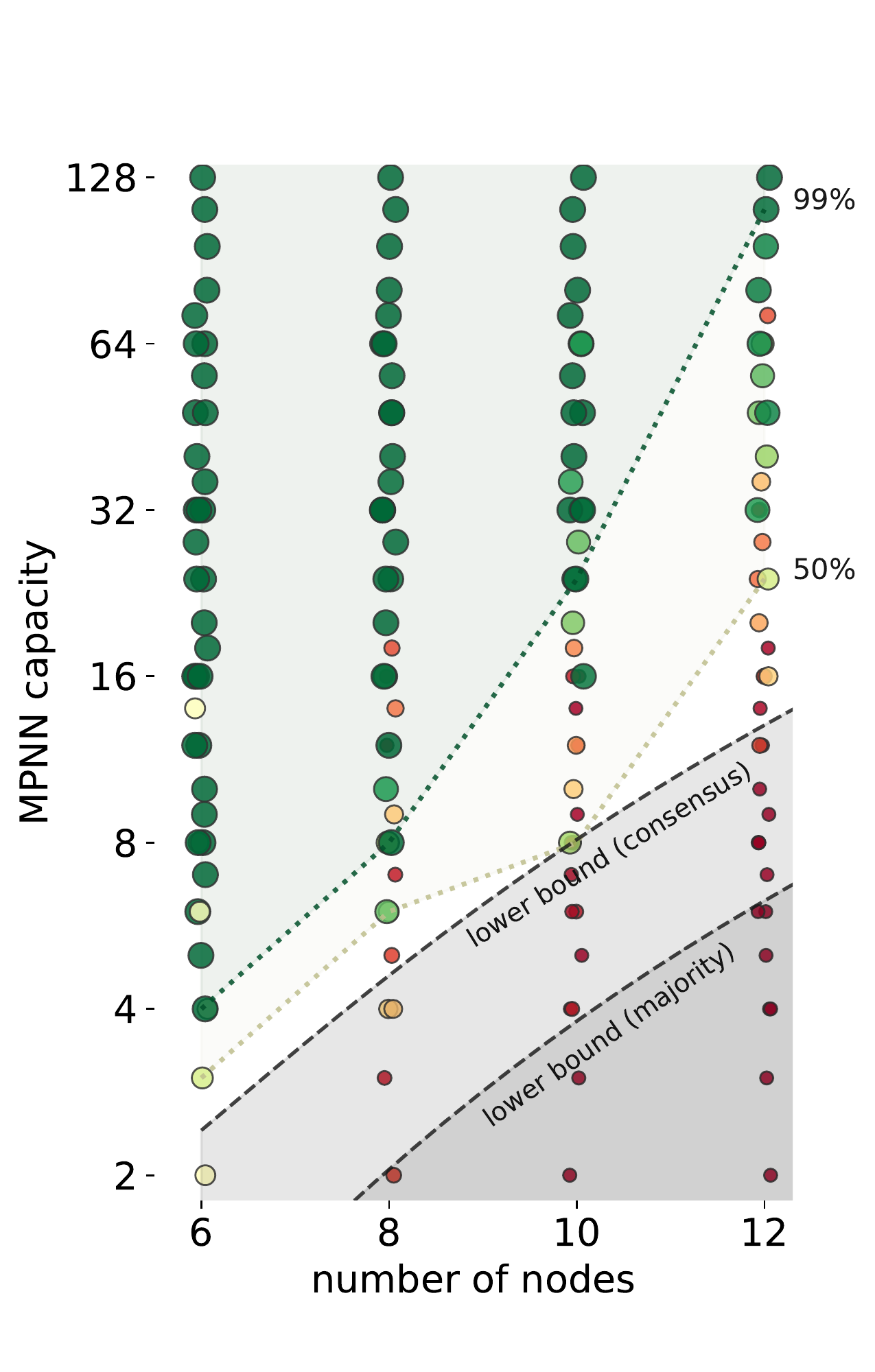}
\vspace{-0.4cm}\caption{distinguishing graphs}
\label{fig:graph_nodes_capacity}
\end{subfigure}
\hfill
\begin{subfigure}[b]{0.60\columnwidth}
\includegraphics[width=1.03\columnwidth,trim=14mm 0 0mm 20mm, clip]{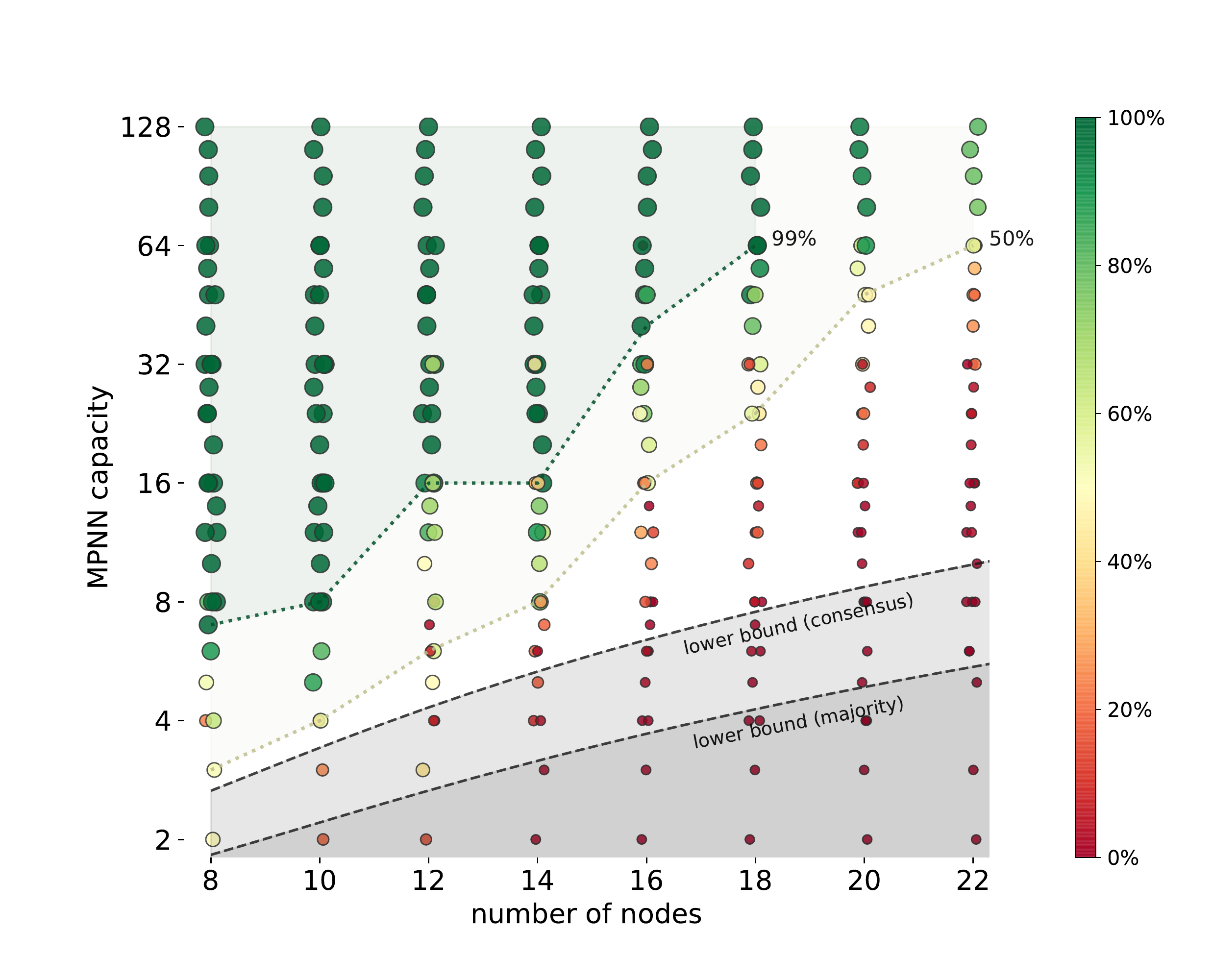}
\vspace{-0.44cm}\caption{distinguishing trees}
\label{fig:tree_nodes_capacity}
\end{subfigure}
\caption{Test accuracy in terms of communication capacity and the number of nodes for 4 graph (left) and 8 tree isomorphism tasks (right). Each marker corresponds to a trained network. Networks of high (low) accuracy as plotted with large green (small red) markers. The two dashed colored lines connect the smallest-capacity networks that attain 50\% and 99\% accuracy, respectively. The two gray regions at the bottom of the figure correspond to the proposed distribution-dependent lower bounds for a majority and consensus readout function. Best seen in color.\vspace{-3mm}}
\end{figure}

\subsection{Findings}

Let me begin by stating that networks of sufficient size could solve nearly every task up to 100\% test accuracy (Table~\ref{tab:same-capacity} in Appendix~\ref{app:experimental_details}), which corroborates previous theoretical findings that non-anonymous MPNN are universal and can solve graph isomorphism~\citep{Loukas2020a,DBLP:journals/corr/abs-1905-12560}, as well as that they can learn to be permutation invariant~\citep{murphy2019relational}. On the other hand, anonymous MPNN are always permutation equivariant but cannot distinguish between graphs of more than three nodes~\cite{chen2020can}.

Figures~\ref{fig:graph_nodes_capacity} and~\ref{fig:tree_nodes_capacity} summarize the neural network performance for all the tasks considered. 
The achieved accuracy strongly correlated with communication capacity (computed based on Lemma~\ref{lemma:maxflow}) with larger-capacity networks performing consistently better. 
Moreover, in qualitative agreement with the analysis, solving a task can be seen to necessitate larger capacity when the number of nodes is increased. A case in point, whereas a capacity of 4 suffices to classify 99\% of graphs of 6 nodes correctly, for 8, 10, and 12 nodes the required capacity increases to 8, 24, and 112, respectively.  
This identified correlation between capacity and accuracy could not be explained by the depth or width of the network alone, as, in most instances, tasks that could not be solved by wide and shallow networks could also not be solved by deep networks of the same capacity. The only exception was when receptive field did not cover the entire graph (see Figures~\ref{fig:graph_exchangable} and \ref{fig:tree_exchangable} in Appendix~\ref{app:experimental_details}). 

The gray regions at the bottom of each figure indicate the proposed expected communication complexity lower bounds. Here, $|\mathcal{S}|=2$ based on the interpretation that each neuron can be either in an activated state or not. 
There are also two lower bounds plotted since a network that sums the final layer's node representations can learn to differentiably approximate both a majority-voting and a consensus function. 
The analysis asserts that a network with capacity below the gray dashed lines should not be able to correctly distinguish input graphs for a significant fraction of all inputs (see precise statement in Lemma~\ref{lemma:mpnn-complexity}).
Indeed, networks in the gray region consistently perform poorly. 
The empirical accuracy appears to match closely the consensus bound, though it remains inconclusive if the network is actually learning to do consensus. 
A closer inspection (see Figures~\ref{fig:graph_all} and~\ref{fig:tree_all} in Appendix~\ref{app:experimental_details}) also reveals that the poor performance of networks in the gray region is not an issue of generalization. In agreement with the theory, networks of insufficient communication capacity do not possess the expressive power to map a fraction of all inputs to the right isomorphism class, irrespective of whether these graphs appear in the training or test set.
%

\section{Conclusion} 

This work proposed a hardness-result for distinguishing graphs in the MPNN model by characterizing the amount of information the nodes can exchange during their forward pass (termed communication capacity). 
From a practical perspective, the results herein provide evidence that, if the amount of training data is not an issue, determining the isomorphism class of graphs is hard but not impossible for MPNN. Specifically, it was argued that the number of parameters needs to increase quadratically with the number of nodes. The implication is that, in the most general case, networks of practical size should be able to solve the problem for graphs with at most a few dozen nodes, but will encounter issues otherwise.

\section*{Broader Impact}

As we rely on neural networks more heavily, we are unfortunately sacrificing some of our ability to understand how our computers solve problems. Our lack of insight hinders us from using our technology to its full potential and can yield mistrust to the public. After all, if we cannot understand what a neural network is (capable of) doing, how can we know whether it is solving the correct problem?
Poor understanding of fundamentals can also lead researchers to misguided optimism, believing that, given the right hyper-parameter tweaking and a large enough training set, neural networks can solve their problem. When incorrect, this mindset can lead to a waste of precious resources, such as time and energy.

In this light, impossibility results, such as those presented in this work, provide an insight into the fundamental limits of neural networks. 
Hardness results for graph neural networks, in particular, characterize the relational pattern recognition ability of practical networks and provide necessary conditions for using our tools to solve classical graph problems. 
The central implication of the results presented in this work is that one cannot expect to learn algorithms that distinguish (even approximately) connected graphs and trees unless the network size grows at-least polynomially with the graph size.

\section*{Acknowledgments}
I would like to express gratitude to the anonymous reviewers, as well as Nathana\"el Perraudin, Nikolaos Karalias and Giovanni Cherubin for their insightful comments. I am also thankful to the Swiss National Science Foundation for financially supporting this work in the context of the project ``\emph{Deep Learning for Graph-Structured Data}'' (grant number PZ00P2 179981).

{
\small
\bibliographystyle{unsrtnat}
\bibliography{references}
}

\appendix

\section{Additional empirical results}
\label{app:experimental_details}

This section presents the empirical results more comprehensively.

First, Table~\ref{tab:dataset_details} provides summary statistics for each of the 12 tasks considered:

\begin{table}[h!]
\centering
\resizebox{\textwidth}{!}{%
\begin{tabular}{@{}rcccccccccccc@{}}
\toprule
\multicolumn{1}{l}{} & $\cX_{\text{graph}}^6$ & $\cX_{\text{graph}}^8$ & $\cX_{\text{graph}}^{10}$ & $\cX_{\text{graph}}^{12}$ & $\cX_{\text{tree}}^{8}$ & $\cX_{\text{tree}}^{10}$ & $\cX_{\text{tree}}^{12}$ & $\cX_{\text{tree}}^{14}$ & $\cX_{\text{tree}}^{16}$ & $\cX_{\text{tree}}^{18}$ & $\cX_{\text{tree}}^{20}$ & $\cX_{\text{tree}}^{22}$ \\ \midrule
classes              & 3                      & 21                     & 231                       & 6328                      & 3                       & 6                        & 21                       & 66                       & 276                      & 1128                     & 5671                     & 22730                    \\
degree (avg.)        & 4.0                    & 4.7                    & 5.4                       & 6.0                       & 3.5                     & 3.6                      & 3.7                      & 3.7                      & 3.8                      & 3.8                      & 3.8                      & 3.8                      \\
diameter (avg.)      & 3.7                    & 4.5                    & 5.0                       & 5.4                       & 4.0                     & 4.3                      & 5.0                      & 5.4                      & 6.0                      & 6.4                      & 6.9                      & 7.3                      \\ 
dataset   size       & 10k                    & 10k                    & 40k                       & 100k                      & 10k                     & 10k                      & 40k                      & 40k                      & 40k                      & 40k                      & 40k                      & 100k                     \\ 
\bottomrule
\end{tabular}%
}
\vspace{2mm}
\caption{Details relevant to the 4 graph and 8 tree isomorphism tasks.}  
\label{tab:dataset_details}
\end{table}

Some example graphs sampled are shown in Figure~\ref{fig:examples}:

\begin{figure}[h!]
\centering
\begin{subfigure}[b]{\examplelength}
\includegraphics[width=\textwidth,trim=0 0 0 0, clip]{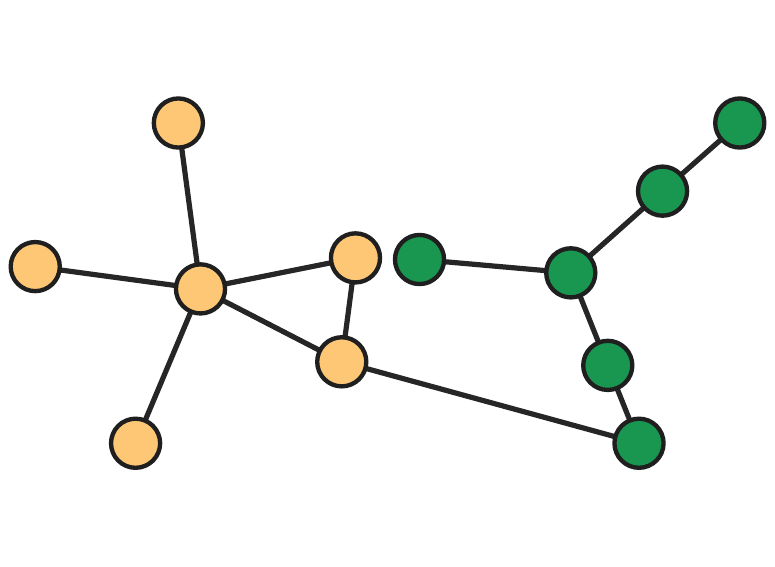}
\end{subfigure}
~
\begin{subfigure}[b]{\examplelength}
   \includegraphics[width=\textwidth,trim=0 0 0 0, clip]{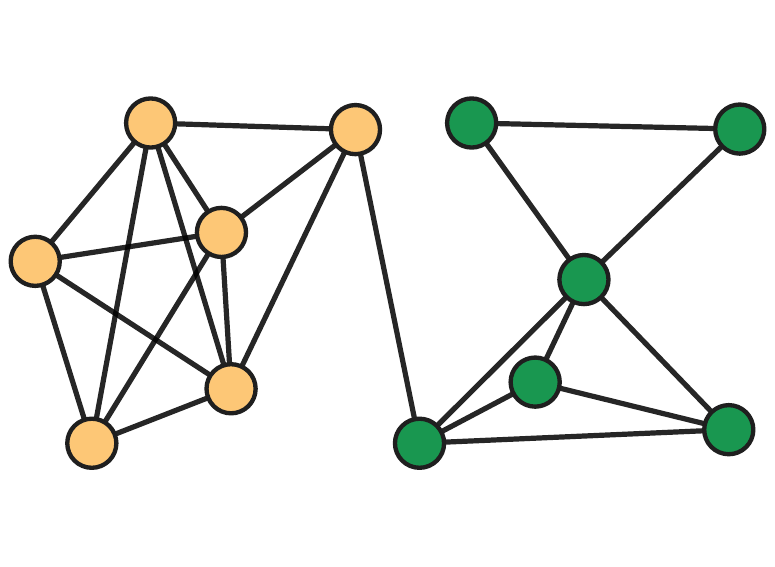}
\end{subfigure}
~
\begin{subfigure}[b]{\examplelength}
\includegraphics[width=\textwidth,trim=0 0 0 0, clip]{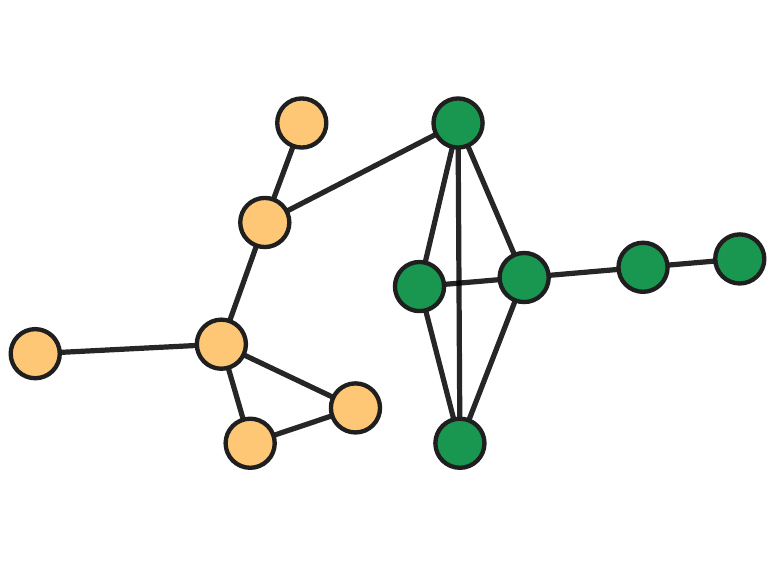}
\end{subfigure}
~
\begin{subfigure}[b]{\examplelength}
    \includegraphics[width=\textwidth,trim=0 0 0 0, clip]{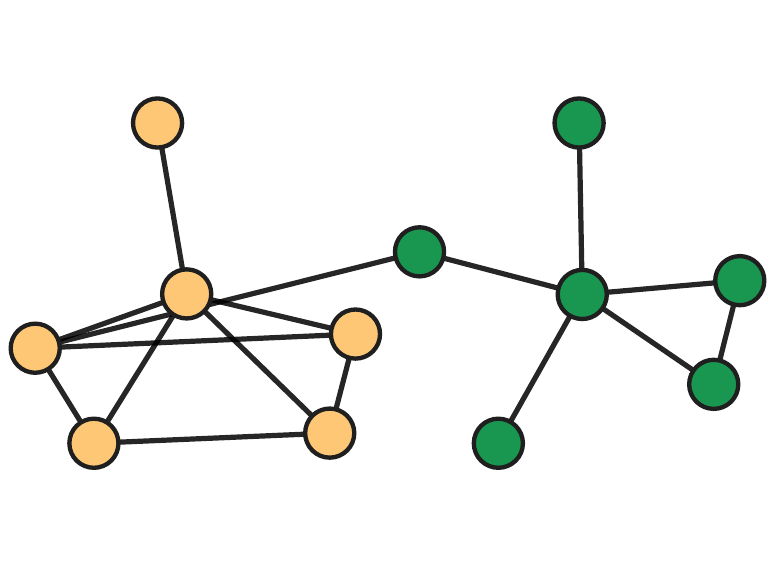}
\end{subfigure}
\\
\begin{subfigure}[b]{\examplelength}
\includegraphics[width=\textwidth,trim=0 0 0 0, clip]{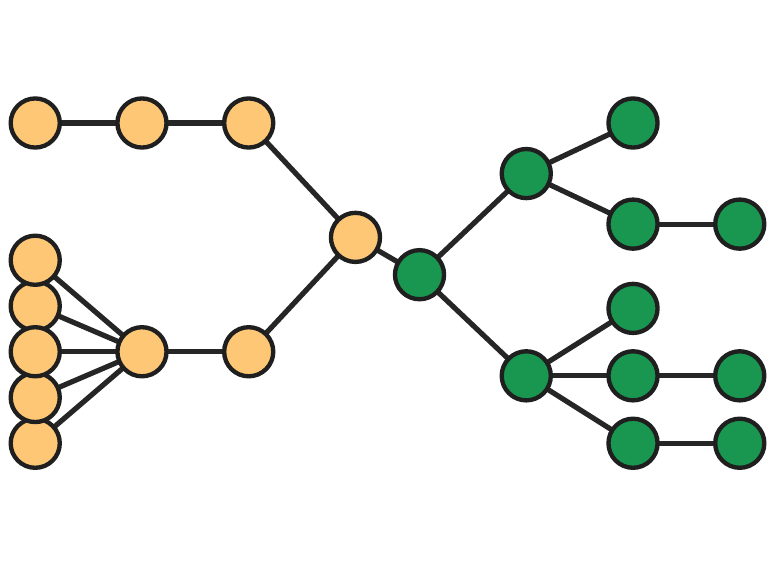}
\end{subfigure}
~
\begin{subfigure}[b]{\examplelength}
    \includegraphics[width=\textwidth,trim=0 0 0 0, clip]{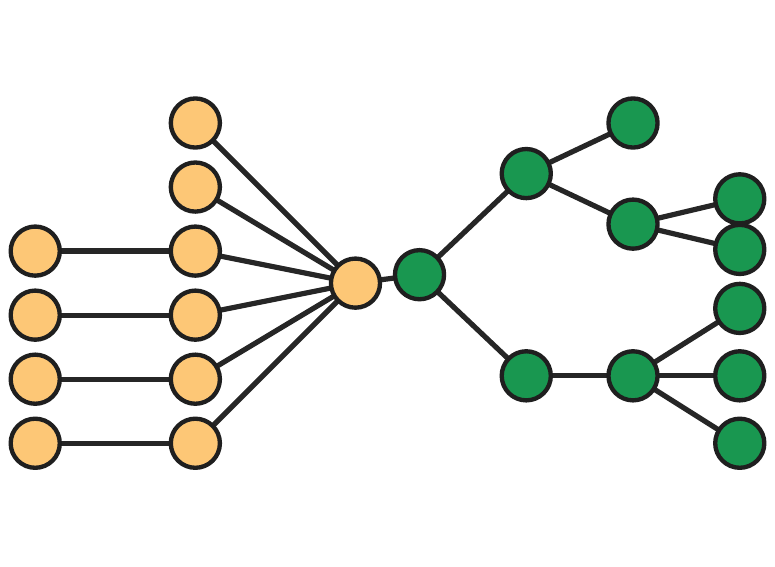}
\end{subfigure}
~
\begin{subfigure}[b]{\examplelength}
\includegraphics[width=\textwidth,trim=0 0 0 0, clip]{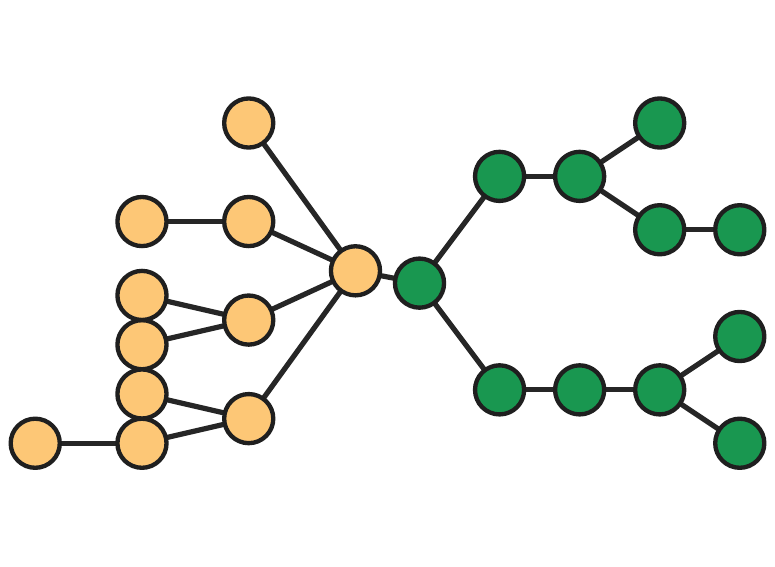}
\end{subfigure}
~
\begin{subfigure}[b]{\examplelength}
    \includegraphics[width=\textwidth,trim=0 0 0 0, clip]{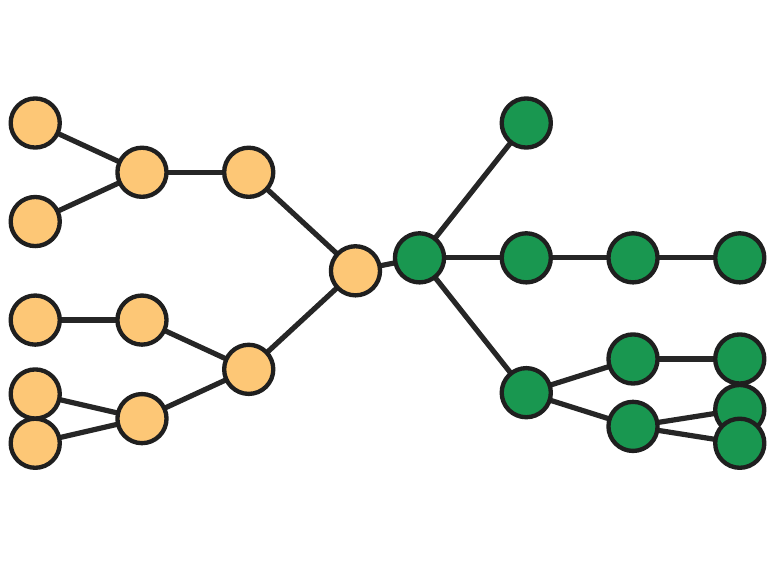}
\end{subfigure}
\caption{Example graphs sampled from two $(\cX_a,\cX_b,1)$ decompositions. Top: $\cX_a$ and $\cX_b$ contain all connected graphs on $v=6$ nodes (special case of Theorems~\ref{theorem:graph_isomorphism_wc} and~\ref{theorem:graph_isomorphism_random}). Bottom: $\cX_a$ and $\cX_b$ contain all trees on 11 nodes (special case of Theorem~\ref{theorem:tree_isomorphism}). In both cases, there exists a $\tau = 1$ cut between the nodes $\cV_a$ controlled by Alice (in yellow) and nodes $\cV_b$ controlled by Bob (in green).}  
\label{fig:examples}
\end{figure}

Table~\ref{tab:same-capacity} provides empirical evidence that, with a one-hot encoding of the node-ordering given as features and a sufficiently large training set, MPNN of sufficient capacity can solve graph isomorphism. In the current experiment, a large network (depth = 10 and width = 32) is seen to solve most isomorphism instances. The network did not achieve perfect classification for larger graphs, but better results can be achieved with more training data.

\begin{table}[h!]
\centering
\resizebox{\textwidth}{!}{%
\begin{tabular}{@{}rcccccccccccc@{}}
\toprule
accuracy   & $\cX_{\text{graph}}^6$ & $\cX_{\text{graph}}^8$ & $\cX_{\text{graph}}^{10}$ & $\cX_{\text{graph}}^{12}$ & $\cX_{\text{tree}}^{8}$ & $\cX_{\text{tree}}^{10}$ & $\cX_{\text{tree}}^{12}$ & $\cX_{\text{tree}}^{14}$ & $\cX_{\text{tree}}^{16}$ & $\cX_{\text{tree}}^{18}$ & $\cX_{\text{tree}}^{20}$ & $\cX_{\text{tree}}^{22}$ \\ \midrule
training   & 100\%                  & 100\%                  & 100\%                     & 99.997\%                  & 100\%                   & 100\%                    & 100\%                    & 100\%                    & 100\%                    & 100\%                    & 100\%                    & 100\%                    \\
validation & 100\%                  & 100\%                  & 100\%                     & 100\%                     & 100\%                   & 100\%                    & 100\%                    & 100\%                    & 100\%                    & 100\%                    & 97.45\%                    & 82.82\%                    \\
test       & 100\%                  & 100\%                  & 100\%                     & 99.96\%                   & 100\%                   & 100\%                    & 100\%                    & 100\%                    & 100\%                    & 100\%                    & 97.35\%                    &  82.92\%                    \\ \bottomrule
\end{tabular}
}
\vspace{2mm}
\caption{The performance of a large-capacity MPNN.}
\label{tab:same-capacity}
\end{table}

The achieved accuracy of all networks considered is shown in Figures~\ref{fig:graph_all} and~\ref{fig:tree_all} for graph and tree isomorphism tasks, respectively. In contrast to the figures of Section~\ref{sec:empirical}, these plots depict the training as well as testing accuracy.  
For the majority of tasks the test and training accuracy is almost identical. Overfitting can be a problem for larger graphs (e.g., trees of at least 20 nodes). The problem can be mitigated by increasing the size of the training set.

\begin{figure}[h!]
\centering
\begin{subfigure}[b]{1\columnwidth}
\includegraphics[width=\textwidth,trim=220 0 200 0, clip]{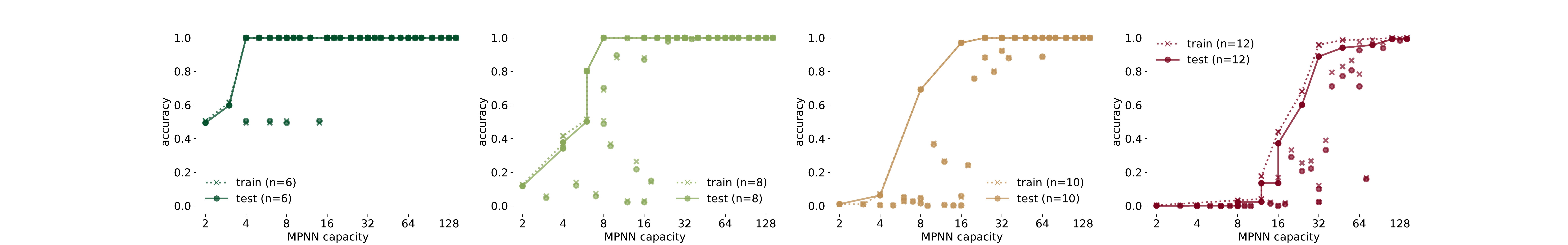}
\vspace{-0.3cm}\caption{graph isomorphism}
\label{fig:graph_all}
\end{subfigure}
\\
\vspace{4mm}
\begin{subfigure}[b]{1\columnwidth}
\includegraphics[width=\textwidth,trim=220 40 200 80, clip]{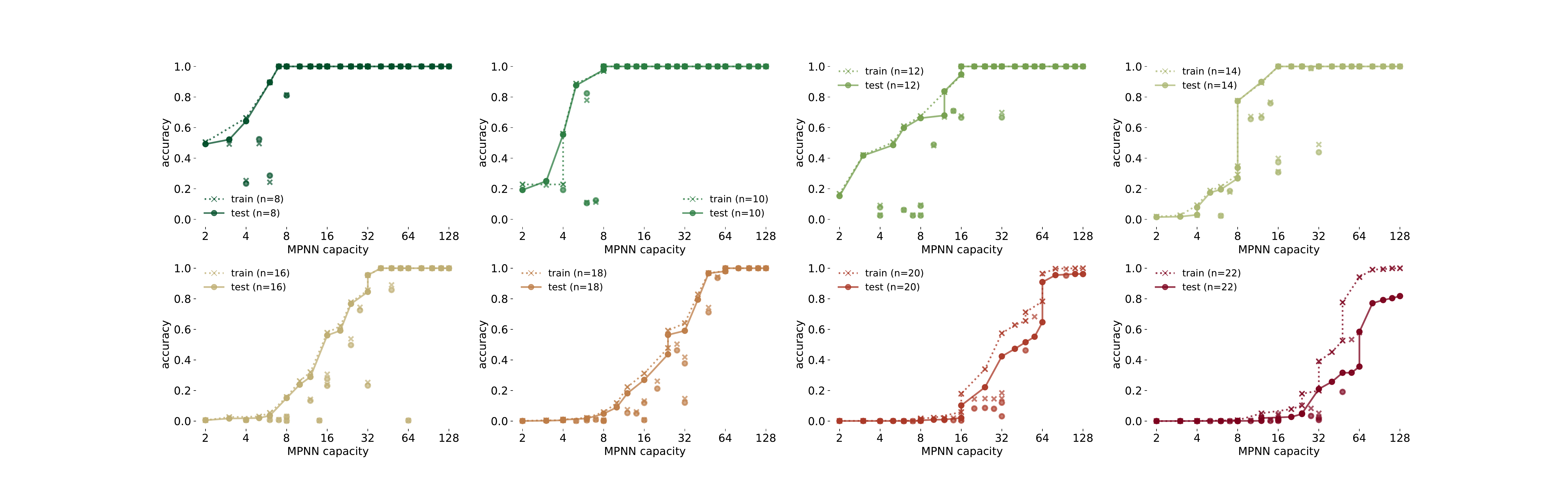}
\vspace{-0.3cm}\caption{tree isomorphism}
\label{fig:tree_all}
\end{subfigure}
\caption{Training and test accuracy as a function of communication capacity.}
\end{figure}

Finally, Figures~\ref{fig:graph_exchangable} and~\ref{fig:tree_exchangable} demonstrate that depth and width are partially exchangeable. This implies that the correlation between capacity and accuracy (see Figures~\ref{fig:graph_nodes_capacity} and~\ref{fig:tree_nodes_capacity}) cannot be explained by only looking at the depth or width of a network. Here, the two figures depict the empirical test accuracy (by the marker color and size) as a function of depth and width for all tasks. For each task, the depth and width have been normalized by the square root of the {critical capacity}, corresponding to the smallest communication capacity of any network that could achieve at least 50\% accuracy. As a consequence of the normalization, all networks in the top-right region  (in white) possess sufficient capacity for the task at hand. Moreover, networks plotted below (above) the main diagonal are deeper than they are wide (wider than they are deep). 
As seen, the classification task can be solved by both wide and deep networks of super-critical capacity, as long as the networks are not too shallow. Indeed, networks of very small depth cannot see the entire graph and thus have poor accuracy.

\begin{figure}[h!]
\centering
\begin{subfigure}[b]{0.48\columnwidth}
\includegraphics[width=0.9\columnwidth,trim=0 0 0 0, clip]{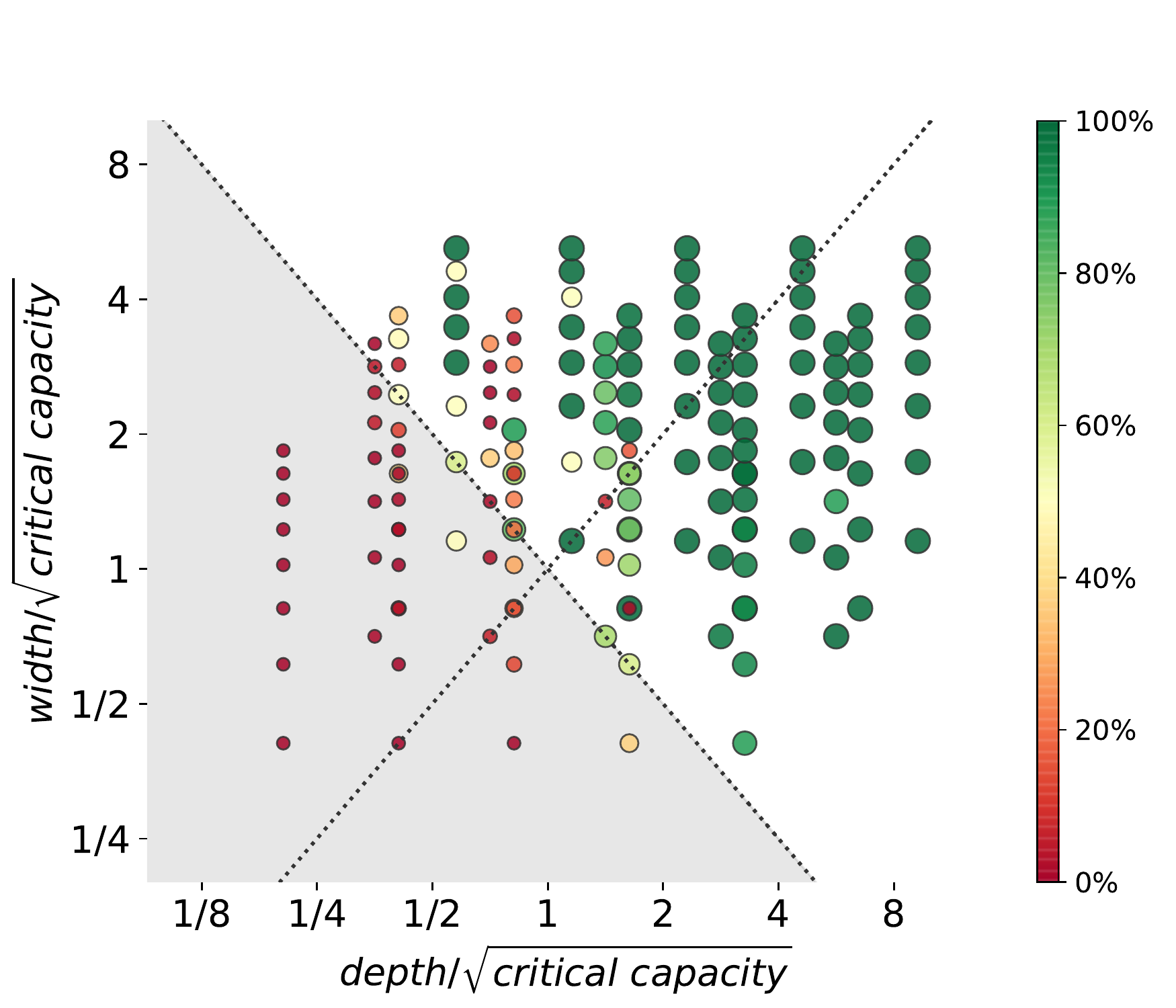}
\vspace{-0.1cm}\caption{graph isomorphism}
\label{fig:graph_exchangable}
\end{subfigure}
\begin{subfigure}[b]{0.48\columnwidth}
\includegraphics[width=0.9\columnwidth,trim=0 0 0 0, clip]{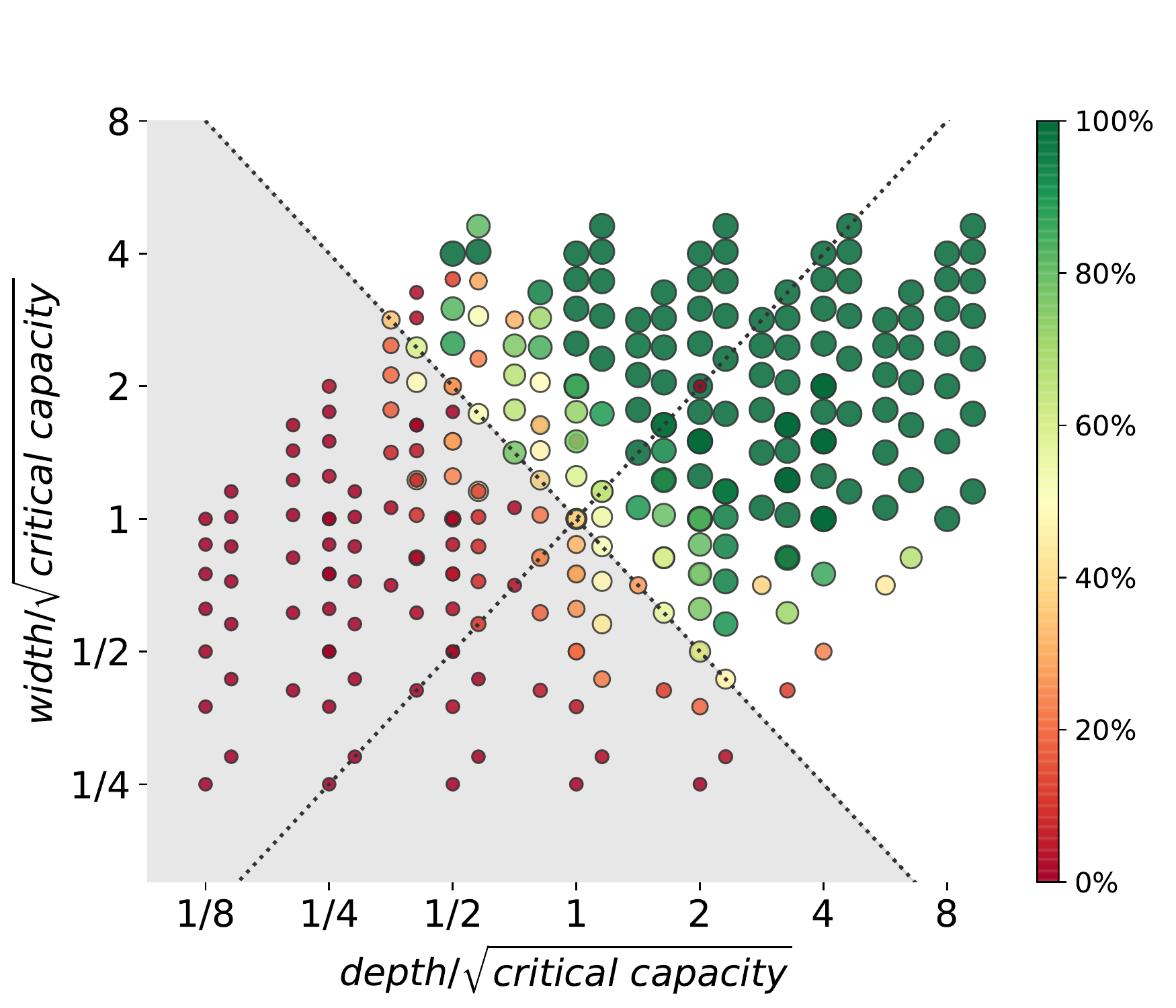}
\vspace{-0.1cm}\caption{tree isomorphism}
\label{fig:tree_exchangable}
\end{subfigure}
\caption{Accuracy as a function of capacity-normalized depth and width. Depth and width are partially exchangeable for graph and tree isomorphism. }
\end{figure}

\section{Communication complexity: basics and beyond}
\label{app:com_complexity}

\subsection{Basic theory: protocols} 

Let us start by denoting by $\cS$ the common set of symbols\footnote{Though usually it is assumed that the parties communicate using binary symbols, i.e., $\cS = \{0,1\}$, the set could also be defined more abstractly to contain $s$ symbols.} Alice and Bob use to communicate and denote by $s = |\cS|$ its cardinality.  A protocol $\pi$ is described in terms of a rooted $s$-ary tree, i.e., a tree with a clearly defined root and in which every internal node has exactly $s$ children. 
In addition, every internal node $i$ is owned by either Alice or Bob and each one of the node's children symbolizes a symbol sent by its owner. Specifically, the protocol associates $i$ with a function $\pi_i$ that maps the input of $i$'s owner to $\cS$ (or equivalently to one of $i$'s children).
The protocol operates as follows: first, both parties set the \textit{current} node to be the root of the tree. Say that the current node is $i$. If the owner of $i$ is Alice then she announces symbol $\pi_i(x)$ and otherwise Bob announces $\pi_i(y)$. Both parties then update the current node to point to the child of $i$ indicated by the value of $\pi_i$. This procedure is repeated until a leaf is found. 
%
\begin{figure}[h!]
\centering
\includegraphics[width=0.25\columnwidth]{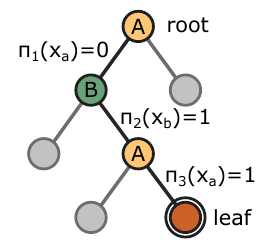}
\caption{The execution of $\pi$ over input $(x_a,x_b)$ is a path within the $s$-ary tree. The decision of which symbol to send is taken by the node's owner (Alice or Bob) as a function of the current path and input. In this example there are $s=2$ symbols $\cS = \{0,1\}$ and the path moves to the left/right child when the symbol 0/1 is sent. The protocol terminates at the leaf and the output is $\pi(x_a,x_b) = ((A,0), (B,1), (A,1))$. }
\end{figure}

By definition, the number of symbols $\|\pi(x_a,x_b)\|_m$ Alice and Bob need to send in order to jointly compute $f(x_a,x_b)$ using protocol $\pi$  equals the length of the path from the root to the leaf $\pi(x_a,x_b)$. Moreover, the number of symbols sent by a protocol in the worst case (i.e., for any input) is at most equal to the depth of the protocol tree (Fact~1.1 in~\citep{rao_behudayoff_2020}).  

\subsection{Basic theory: monochromatic rectangles} 
\label{subsec:basic_theory_monochromatic}

To understand how protocols operate one needs to consider the concept of rectangles. A \textit{rectangle} is a subset of $\cX_a \times \cX_b$ that can be expressed as $\cX_a' \times \cX_b'$ for some $\cX_a' \subset \cX_a$ and $\cX_b' \subset \cX_b$. Intuitively, if one represents $\cX_a \times \cX_b$ as a matrix $\bm{X}$ of size $|\cX_a| \times |\cX_a|$, then a rectangle is any principal submatrix $\bm{X}'$ of $\bm{X}$, i.e., a matrix that contains a subset of rows and columns.

As it turns out, every protocol can be described in terms of rectangles. Let $\cal{R}^i \subseteq \cX_a \times \cX_b$ be the set of inputs leading to a path that crosses a node $i \in \pi$. Moreover, define the following sets: 
\begin{align*}
    \cX_a^i &= \{x \in \cX_a: \exists y \in \cX_b \text{ such that } (x_a,x_b) \in \cal{R}^i \} \\
    \cX_b^i &= \{y \in \cX_b: \exists x \in \cX_a \text{ such that } (x_a,x_b) \in \cal{R}^i \}
\end{align*}
The following result clarifies the connection between protocols and rectangles.
\begin{lemma}[Lemma~1.4 in~\citep{rao_behudayoff_2020}]
For every protocol $\pi$ and node $i$, $\cal{R}^i$ is a rectangle with $\cal{R}^i = \cX_a^i \times \cX_b^i$. Further, the rectangles $\cal{R}^\ell$ given by the leafs $\ell \in \cL_{\pi}$ of the protocol tree form a partition of $\cX_a \times \cX_b$. 
\label{lemma:rectangles}
\end{lemma}
Effectively, at any point in a protocol, a rectangle describes the different possible outputs of $f$ given the messages that have been exchanged. Every new message that the two parties exchange, eliminates some possible outputs, decreasing the size of the rectangle.

With this in place, it is not hard to realize that, for every leaf $\ell \in \cL_{\pi}$, the function $f$ should always take the same value at every $(x_a,x_b) \in \cal{R}^\ell$ in order for both parties to be able to compute the output from $\pi(x_a,x_b)$. Such rectangles are referred to as \textit{monochromatic}: concretely, a rectangle $\cal{R} \subset \cX_a \times \cX_b$ is monochromatic if $f(x_a,x_b) = f(x_a',x_b')$ for every $(x_a,x_b), (x_a',x_b') \in \cal{R}$. Indeed, if leaf rectangles were not monochromatic, Alice and Bob would not be able to identify the output of $f$ based on $\cal{R}^{\ell}$.   

The following theorem is obtained by combining Lemma~\ref{lemma:rectangles} with the fact that the minimum depth of any $s$-ary tree with $s^{c}$ leafs is $c$.     

\begin{theorem}[Theorem 1.6 by~\citet{rao_behudayoff_2020}] If the communication complexity of $f: \cX_a \times \cX_b \to \cY$ is $c_f^\textit{both}$, then $\cX_a \times \cX_b$ can be partitioned into at most $s^{c_f^\textit{both}}$ monochromatic rectangles. 
\end{theorem}
The following is a direct corollary: 
\begin{corollary}[\citet{rao_behudayoff_2020}]
If $\cX_a \times \cX_b$ cannot be partitioned into $s^c$ monochromatic rectangles, then $c_f^\textit{both} \geq c$. 
\label{cor:max_rectangle_bound}
\end{corollary}
A simple way to satisfy the requirement of the corollary is to prove that no large monochromatic rectangle exists. For instance, if it is shown that all monochromatic rectangles have size bounded by $k^2$ then every monochromatic partitioning must contain at least $|\cX_a \times \cX_b|/k^2$ rectangles and the complexity is at least $c_f^\textit{both} \geq \log_s{\left( |\cX_a \times \cX_b|/k^2 \right)}$. I will rely on this method in the following to derive lower bounds on the worst-case communication complexity of different functions.

\subsection{A different perspective: expected communication complexity}
\label{subsec:expected_lower_bound}

The following lemma connects the expected communication complexity $\Edist{\bbD}{c_f(\pi)}$ to the entropy of the categorical distribution induced by the leafs of the protocol tree.  

\begin{lemma}
Let the random variables $X = (X_a,X_b) \sim \bbD$ be sampled from some distribution $\bbD$ and, moreover, suppose that the random variable $L_\pi$ is the leaf for a protocol $\pi$ that computes $f(X_a,X_b)$. The expected communication complexity of $f$ is 
$$
    \min_{\pi}  \entropy{s}{L_{\pi}}  \leq c_f^m({\bbD}) \leq \min_{\pi} \entropy{s}{L_{\pi}} + 1,   
$$
where $\entropy{s}{L_{\pi}}$ is the Shannon entropy (base $s$) of $L_{\pi}$ under $\bbD$.
\label{lemma:leaf_entropy_lower_bound}
\end{lemma}
\begin{proof}
The expected length of a protocol $\pi$ is  
\begin{align*}
    \Edist{\bbD}{c_f^m(\pi)} 
    &= \sum_{x_a, x_b} \|\pi(x_a,x_b)\|_m \cdot \Prob{X_a=x_a, X_b=x_b} \\
    &= \sum_{\ell \in \cL_{\pi}}  \text{depth}(t) \cdot \Prob{L_{\pi}=\ell} \\
    &= \Edist{\bbD}{\text{depth}(L_{\pi})}.    
\end{align*}
Note that the set $\cL_{\pi}$ contains the leafs of the protocol tree and $L_{\pi}$ is a categorical random variable over leafs with
$$
\Prob{L_{\pi} = \ell} = \sum_{x,y \ : \  \pi(x_a,x_b) = \ell} \Prob{X_a=x_a, X_b=x_b},
$$
which is also equal to the probability $\Prob{(X_a,X_b) \in \cal{R}^\ell}$ that a randomly drawn input belongs to $\cal{R}^\ell$.  

To understand $c_f^m({\bbD})$ it helps to realize the connection between protocols and coding theory: rather than sending information between Alice and Bob, one may think of sending the leafs over a channel by using a codebook. In this analogy, each leaf corresponds to a code and the path from the root of the protocol tree to every internal node at depth $t$ corresponds to code prefix of length $t$. Furthermore, the probability of encountering the leaf is $P(L_{\pi}=t)$ and the depth of the protocol tree for every input $(x_a,x_b) \in \cal{R}^\ell$ is equal to the length of the code required to send the associated symbol. 

From the above it follows that the act of designing a protocol with minimal $c_f^m(\pi)$ is equivalent to finding a tree with minimum expected path length from the root to the leafs. The latter is, in turn, equivalent to minimizing the length of the expected code length for a categorical distribution $L_{\pi}$. Therefore, based on Shannon's source coding theorem we have that 
$$
    \min_{\pi} \entropy{s}{L_{\pi}}  \leq c_f^m({\bbD}) \leq \min_{\pi} \entropy{s}{L_{\pi}} + 1,   
$$
matching the lemma statement.
\end{proof}

\section{Deferred proofs}

\subsection{Proof of Lemma~\ref{lemma:maxflow}}

\begin{figure}[h!]
\centering
\hspace{-11mm}
  \begin{subfigure}[b]{0.45\textwidth}
\includegraphics[width=\textwidth,trim=1 0 1 0, clip]{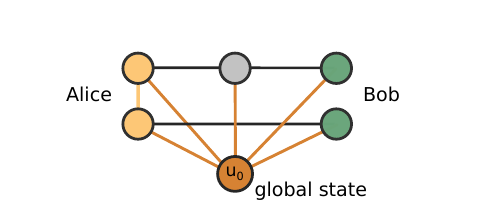}
    \caption{Example partitioning}
    \label{fig:1}
  \end{subfigure}
  \begin{subfigure}[b]{0.45\textwidth}
    \includegraphics[width=\textwidth,trim=1 0 1 0, clip]{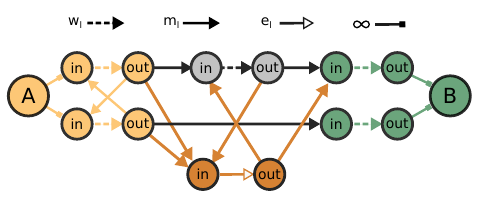}
    \caption{Maximum flow reduction}
    \label{fig:2}
  \end{subfigure}
  \caption{An example of the reduction employed in the proof of Lemma~\ref{lemma:maxflow}. The yellow and green subgraphs correspond respectively to $G_a$ and $G_b$. The global state (external memory) is shown in orange. Each edge is annotated based on its capacity in the maximum flow reduction. }
\end{figure}

The number of symbols that can be transmitted from Alice to Bob in layer $\ell$ is bounded by the maximum flow of the following \textit{multi-source multi-sink maximum flow problem with node capacities}: 
\begin{itemize}
\item The nodes $\cV_a$ are the senders and the nodes $\cV_b$ are the sinks.
\item Each edge has capacity $m_{\ell}$.
\item Each node in $\cV$ has capacity $w_\ell$, whereas $v_0$ has capacity $\gamma_{\ell}$.
\end{itemize}
This problem can be reduced to a simple maximum flow problem (single source single-sink without node capacities) in three steps: 
\begin{enumerate}
\item All nodes in $\cV_a$ (resp. $\cV_b$) are connected to a new node $A$ (resp. $B$) with edges of infinite capacity. 

\item Each node $v_i$ (with the exception of $A,B$ and $v_0$) is split into two nodes $\text{in}_i$ and $\text{out}_i$ connected by an edge of capacity $w_\ell$. Incoming edges to $v_i$ are connected to $\text{in}_i$ and outgoing edges are connected to $\text{out}_i$. 

\item The same splitting procedure is performed for node $v_0$, but now the internal edge has capacity $\gamma_{\ell}$.
\end{enumerate}

Consider the transformed flow network as shown in Figure~\ref{fig:2}. By the max-flow min-cut theorem, the maximum value of the flow is equal to the minimum capacity over all cuts that separate $\cV_a \cup A$ from $\cV_b \cup B$. The latter however can always be bounded by $\text{cut}(A, B) + \gamma_{\ell}$. The first term of this equation gives the weight of the smallest cut separating $A$ and $B$ in the reduced graph, excluding those (orange) edges that touch $v_0$: since the edges from $A$ to $\cV_a$ have infinite capacity (resp. from $B$ to $\cV_b$), every such cut also separates $\cV_a$ and $\cV_b$. Notice also that every path from $A$ to $B$ includes at least one internal edge of capacity $w_{\ell}$ and one normal edge of capacity $m_{\ell}$. Combining the previous observations one finds that $\text{cut}(A, B) \leq \text{cut}(\cV_a, \cV_b) \min\{w_{\ell},  m_{\ell} \}$, where $\text{cut}(\cV_a, \cV_b)$ is the size of the smallest cut that separates $\cV_a$ and $\cV_b$ on $G$ (the undirected and unweighted graph prior to the reduction). The internal edge capacity of $v_0$ in accounted by term $\gamma_{\ell}$.
The final expression is obtained by summing the bound over all $d$ layers.

\subsection{Proof of Theorem~\ref{theorem:graph_isomorphism_wc}}

The proof consists of two main steps. First, the number of monochromatic rectangles of $f_{\text{isom}}$ will be controlled using the number of graph isomorphism classes in $\cX$. Then, invoking Corollary~\ref{cor:max_rectangle_bound} will result in a bound for $c_{f_\text{isom}}^\textit{both}$. Second, the identified lower bound will be translated to a bound regarding $c_{f_\text{isom}}^\textit{one}$ based on Lemma~\ref{lemma:c_single}.

There are $2^{{\binom{v}{2}}}$ labeled graphs on $v$ nodes (i.e., counting orderings), the overwhelming majority of which are connected. The number of connected labeled graphs on $v$ nodes is 
    $$
    |\cX_a| = |\cX_b| = 2^{{\binom{v}{2}}} \left( 1 - \frac{2v}{2^v} + o\left(\frac{1}{2^v} \right) \right) = 2^{{\binom{v}{2}}} \left( 1 - O\left(\frac{v}{2^v}\right) \right) ,
    $$
   which, for sufficiently large $v$, is very close to $2^{{\binom{v}{2}}}$~\citep[p. 138]{flajolet2009analytic}.
Specifically, one may write 
\begin{align*}
\log_2{|\cX_a|} = \log_2{|\cX_b|}
    &= \log_2{ \left( 2^{{\binom{v}{2}}} \left( 1 - O\left(\frac{v}{2^v}\right) \right) \right) } \\
    &= \binom{v}{2} \log_2{2} + \log_2{ \left( 1 - O\left(\frac{v}{2^v}\right) \right)} \\ 
    &\geq \frac{v(v-1)}{2}  - O\left(\frac{v}{2^v}\right) \tag{$\log(1-x)\geq -O(1)x$ for $x = o(1)$} \\
    &= \frac{v(v-1)}{2} + o(1)
\end{align*}
and, similarly, $\log_2{|\cX_a|} = \log_2{|\cX_b|} \leq \frac{v(v-1)}{2}$.
The number of permutations on $v$ nodes is $v!$, which implies that the number $c(v)$ of isomorphism classes of $v$-node graphs is bounded by 
\begin{align}
	\log_{2}{c(v)} 
	&\geq \log_{2}{\left(\frac{|\cX_a|}{v!}\right)} \\
	&= \frac{v(v-1)}{2} - \log_2{(v!)} + o(1)  \\
	&\geq \frac{v(v-1)}{2} - v\log_2{\left(\frac{v}{e}\right)} - \log_2{\left(\sqrt{ve^2}\right)} + o(1) 
	\tag{ since $x! \leq \sqrt{xe^2} \, (x/e)^{x}$} \\
	&= \frac{v^2}{2} - v\log_2{\left(\frac{v\sqrt{2}}{e}\right)} - \log_2{\left(\sqrt{ve^2}\right)} + o(1)	
	\label{eq:graph_isom_wc_gv}
\end{align}
By construction, $\cX$ contains at least $c(v)(1+c(v))/2$ classes. To obtain this bound, one assumes that there do not exist any classes that differ only w.r.t. $G_c$ and then notes that each unique class of $\cX$ may be build either by gluing two distinct or identical classes on $v$ nodes (corresponding to graphs in $\cX_a$ and $\cX_b$). The bound then follows by counting all pairs of elements (there are $c(v)$ of those) with repetitions (e.g., for $\{a,b,c\}$ the set of possible pairs are $\{(aa),(ab),(ac),(bb),(bc),(cc)\}$).

The number of monochromatic rectangles of $f_\text{isom}$ is at least the number of classes and thus Corollary~\ref{cor:max_rectangle_bound} asserts:
\begin{align}
c_{f_\text{isom}}^\textit{both} \log_2{s} 
	&= \log_2{ \left(
    \left\{\begin{array}{c}
        \text{minimum number of} \\
        \text{monochromatic } \\
        \text{rectangles}
      \end{array}
    \right\} \right)} \notag \\
    &\geq \log_2{ \left( \frac{c(v) (c(v)+1)}{2} \right) } \notag \\ 
    &=  2 \log_2{c(v)} + \log_2{\left(1+\frac{1}{c(v)}\right) } - 1 
    \geq 2 \log_2{c(v)} - 1
	\label{eq:graph_isom_wc_both_1}
\end{align} 
Substituting~\eqref{eq:graph_isom_wc_gv} into \eqref{eq:graph_isom_wc_both_1} gives:
\begin{align*}
c_{f_\text{isom}}^\textit{both} \log_2{s}
	&\geq v^2 - 2v\log_2{\left(\frac{v\sqrt{2}}{e}\right)} - 2 \log_2{\left(\sqrt{ve^2}\right)} -1 + o(1) \\
	&= v^2 - 2v\log_2{\left(\frac{v\sqrt{2}}{e}\right)} - \log_2{\left(2 ve^2\right)} + o(1)
\end{align*} 
A bound on $c_{f_\text{isom}}^{\textit{one}}$ can be derived with the help of Lemma~\ref{lemma:c_single}:
\begin{align*}
	c_{f_\text{isom}}^{\textit{one}} \log_2{s}
	&\geq c_{f_\text{isom}}^\textit{both} \log_2{s} - \max_{G_b,G_c} \log_s{\left(|\{ f(G_a,G_b,G_c) \, : \, G_a \in \cX_a \} |\right)} \log_2{s} \\
	&= 2 \log_2{c(v)} - 1  - \log_2{c(v) } \\
	&\geq \frac{v^2}{2} - v\log_2{\left(\frac{v\sqrt{2}}{e}\right)} - \log_2{\left(2e\sqrt{v}\right)} + o(1).	
\end{align*}
This lower bound derivation finishes by factoring $c_{f_\text{isom}}^{\textit{one}}$ as a function of $c_{f_\text{isom}}^{\textit{both}}$.

To conclude the proof, one notes the following elementary upper bound: to compute $f_\text{isom}(G)$,  Bob and Alice can simply send their entire edge-sets to each other and proceed to compute $f(G_a, G_b, G_c)$ independently. Then, since the number of edges of a graph $v$ nodes are $|\cE_a|,|\cE_b| \leq v(v-1)/2$, it suffices to exchange  
$ 
    c_{f_\text{isom}} \leq {v(v-1)}/\log_2{s} = O(v^2)
$  
symbols.

\subsection{Proof of Theorem~\ref{theorem:graph_isomorphism_random}}

I will begin by proving a more general result. Specifically, it will be shown that the expected communication complexity is directly bounded by the entropy of the isomorphism class of a graph sampled from $\bbG$.

\begin{lemma}
The expected number of symbols that Alice and Bob need to exchange to jointly compute the isomorphism class $f_\text{isom}(G)$ of a graph sampled from $G = (G_a,G_b,G_c) \sim \bbG$ is at least
\begin{align*}
    c_{f_\text{isom}}^\textit{both}(\bbG) \geq \min_{G_c} \entropy{s}{f_\text{isom}(G)|G_c}. 
\end{align*}
\label{lemma:entropy_lower_bound}
\end{lemma}
\begin{proof}
The first step is to condition the expected communication complexity on $G_c$:  
\begin{align}
    c_{f_\text{isom}}(\bbG) 
    &= \min_{\pi} \Edist{{\bbG}}{c_{f_\text{isom}}(\pi)} \notag \\
    &=  \min_{\pi} \sum_{ G_c }  \Prob{ G_c } \Edist{{\bbG}}{c_{f_\text{isom}}(\pi) | G_c} \tag{due to the law of total expectation} \\
    &=  \min_{\pi} \sum_{ G_c }  \Prob{ G_c }  \Edist{{\bbG}}{c_{f_{c}}(\pi)} \tag{by the definition $f_{c}(\cdot, \cdot) \defeq f_\text{isom}(\cdot, \cdot, G_c) $}  \\
    &\geq  \sum_{ G_c }  \Prob{ G_c }  \min_{\pi}  \Edist{{\bbG}}{c_{f_c}(\pi)} \geq  \min_{G_c} c_{f_{c}}(\bbG). \notag
\end{align}
Denote by $\cL_{\pi}$ the set of leafs of a protocol $\pi$ that computes $f_c$ and by $L_{\pi}$ the random variable induced by the distribution $\bbG$ (for brevity, the conditioning on $G_c$ remains implicit in the following).  We have that
\begin{align}
    \entropy{s}{L_{\pi}} 
    &= \sum_{\ell \in \cL_{\pi}} \Prob{L_{\pi} = \ell} \log_s{ \left(\frac{1}{\Prob{L_{\pi} = \ell}} \right)}.
\end{align}
Upon closer consideration, there are $|\cY|$ types of leafs such that $\cL_{\pi} = \bigcup_{y = 1}^{|\cY|} \cL_{\pi,y}$, with each subset $\cL_{\pi}^l$ containing all leafs for which the protocol outputs the graph isomorphism class $y$. From Lemma~\ref{lemma:entropy_partitioning_ineq} and because $\cL_{\pi,1}, \ldots, \cL_{\pi,|\cY|}$ form a partitioning of $\cL_{\pi}$, we may write:
\begin{align*}
    \entropy{s}{L_{\pi}} 
    &\geq \sum_{y = 1}^{|\cY|} \Prob{L_{\pi} \in \cL_{\pi,y}} \log_s{ \left(\frac{1}{\Prob{L_{\pi} \in \cL_{\pi,y}}} \right)}.    
\end{align*}
The term $\Prob{L_{\pi} \in \cL_{\pi,y}}$ seen above corresponds to the probability that class $y$ will appear in our sample: 
$$
    \Prob{L_{\pi} \in \cL_{\pi,y}} = \Prob{f(G_a,G_b,G_c) = y}
$$
therefore, $\min_{\pi} \entropy{s}{L_{\pi}} \geq \entropy{s}{f(G)|G_c}$ and the claim follows.
\end{proof}

Coming back to the setting of the main theorem, denote by $k_y = |\cE_a| + |\cE_b|$ the number of edges of the graphs in class $y$ (disregarding the edges $\cE_c$). For every $G_c$, we have that 
\begin{align*}
    \Prob{ f_\text{isom}(G) = y \, | \, G_c } = i_c(v) \, p^{k_y} (1-p)^{2\binom{v}{2} - k_y} = i_c(v) \, p^{k_y} (1-p)^{v(v-1) - k_y}.
\end{align*}
Term $i_c(v)$ corresponds to the size of the corresponding isomorphism class. 
Specifically, when $p$ is not too small and $\text{cut}(\cV_a, \cV \setminus \cV_a) = \text{cut}(\cV_b, \cV \setminus \cV_b) = 1$, it can be inferred that each isomorphism class in the universe contains at most $ 2 (v!)^2$ labeled graphs. The remaining $n!-2(v!)^2$ permutations yield isomorphic graphs with cut larger than one. 

\begin{claim}
For any $\delta>0$, $\text{cut}(\cV_a, \cV \setminus \cV_a) = \text{cut}(\cV_b, \cV \setminus \cV_b) = 1$, and $p \geq (\delta + \log_v)/v$, we have $ i_c(v)  \leq 2 (v!)^2$ with probability at least $e^{-2e^{-\delta}} + o(1)$.
\label{claim:i_c_v}
\end{claim}
\begin{proof}
To see this consider a labeled graph $G \in \cX$ and let $G' = (\cV', \cE')$ be a second labeled graph that is isomorphic to $G$, induced by a the label permutation $\cV' = (\Pi(u) : u \in \cV)$. I claim that, if there exist $ v_i,v_j \in \cV_a$ for which $\Pi(v_i) \in \cV_a$ and $\Pi(v_j) \in \cV_b$, then $G' \notin \cX$ (and the same holds if there exist $ v_i,v_j \in \cV_b$ for which $\Pi(v_i) \in \cV_b$ and $\Pi(v_j) \in \cV_b$).  

The claim is proven by contradiction: suppose (for now) that $G_a$ and $G_b$ are connected. Then, for every set $\cS$ of cardinality $v$ that is a strict subset of \textit{both} $\cV_a$ and $\cV_b$ ($\cS$ corresponds to the nodes with labels $(1, \cdots, v)$ in $G'$) the cut between $\cS$ and its complement must be $\text{cut}(\cS, \cV \setminus \cS) = \sum_{v_i,v_j} \{ v_i \in \cS \text{ and } v_j \notin \cS \} = \sum_{v_i,v_j} \{ v_i \in \cS \text{ and } v_j \in ( \cV_a \setminus \cS) \} + \sum_{v_i,v_j} \{ v_i \in \cS \text{ and } v_j \in ( \cV_b \setminus \cS) \}\geq 1 + 1$. The latter, however, is impossible as we have assumed that $\forall G' \in \cX$, we must have $\text{cut}(\cV_a', \cV' \setminus \cV_a') = \text{cut}(\cV_b', \cV' \setminus \cV_b') = 1$. Therefore, the only valid permutations $\Pi$ are those that abide to either (a) if $v_i \in \cV_a \rightarrow  \Pi(v_i) \in \cV_a$ and if $v_i \in \cV_b \rightarrow  \Pi(v_i) \in \cV_b$ (there are $(v!)^2$ such permutations), or (b) if $v_i \in \cV_a \rightarrow  \Pi(v_i) \in \cV_b$ and if $v_i \in \cV_a \rightarrow  \Pi(v_i) \in \cV_a$ (there are $(v!)^2$ such permutations).

In the studied distribution, there is a non-zero probability that a disconnected graph appears. However, the probability is exponentially small when $p>{\log{v}}/{v}$. It is well known (see e.g., Theorem 4.1 by~\citet{frieze2016introduction}) that, for any $\delta > 0$ and $p = \frac{\delta + \log{v}}{v} $, a random graph on $v$ nodes is connected with probability
$$
    \Prob{G_a \text{ is connected}} = \Prob{G_b \text{ is connected}} = e^{-e^{-\delta}} + o(1) 
$$
and, by independence, 
$
    \Prob{G \text{ is connected}} = e^{-2e^{-\delta}} + o(1). 
$
\end{proof}

Based on the above observation, the conditional entropy of $f(G)$ can be rewritten as  
\begin{align*}
    \entropy{2}{f_\text{isom}(G) | G_c} 
	&= \sum_{y \in \cY}  \Prob{ f_\text{isom}(G) = y | G_c } \log_2{\left( \frac{1}{\Prob{ f_\text{isom}(G) = y | G_c }}\right) } \\ 
    &\hspace{-21mm}\geq \sum_{k = 0}^{v(v-1)} \frac{\binom{v(v-1)}{k}}{i_c(v)} i_c(v) \, p^{k} (1-p)^{v(v-1) - k_y} \log_2{ \left(\frac{1}{i_c(v) \, p^{k} (1-p)^{v(v-1) - k}} \right)} \\
    &\hspace{-21mm}= \sum_{k = 0}^{v(v-1)} \binom{v(v-1)}{k}  p^{k} (1-p)^{v(v-1) - k} \left( -\log_2{i_c(v)} + v(v-1)\log_2{\left( \frac{1}{1-p}\right)} + k \log_2{ \left(\frac{1-p}{p} \right)} \right) \\
    &\hspace{-21mm}= \log_2{ \left(\frac{1-p}{p} \right)} \left( \sum_{k = 0}^{v(v-1)} \binom{v(v-1)}{k}  p^{k} (1-p)^{v(v-1) - k} k  \right) +  v(v-1)\log_2{\left( \frac{1}{1-p}\right)} -\log_2{i_c(v)} 
\end{align*}
Let $B$ be a binomial random variable with parameters $v(v-1)$ and $p$. The summation term is equivalent to the expectation of $B$:
\begin{align*}
    \sum_{m = 0}^{v(v-1)} \binom{v(v-1)}{k}  p^{k} (1-p)^{v(v-1) - k} k = \E{B} = v(v-1)p
\end{align*}
and, therefore,
\begin{align*}
    \entropy{2}{L_{\pi}} 
    &\geq \log_2{ \left(\frac{1-p}{p} \right)} v(v-1)p +  v(v-1)\log_2{\left( \frac{1}{1-p}\right)} -\log_2{i_c(v)} \\
    &= v(v-1) \entropy{2}{p} -\log_2{i_c(v)}  \tag{by definition $\entropy{2}{p} =   \log_2{ \left(\frac{1-p}{p} \right)} p + \log_2{\left( \frac{1}{1-p}\right)} $} \\
    &= v(v-1) \entropy{2}{p} - 2 \log_2{v!} -1  \tag{see Claim~\ref{claim:i_c_v} $i_c(v) \leq 2 (v!)^2$ } \\
    &\geq v(v-1) \entropy{2}{p} - 2 \left( v \log_2{\left(\frac{v}{e} \right)}  + \frac{1}{2}\log_2{\left( ve^2 \right) } \right) - 1 \tag{ since $x! \leq \sqrt{xe^2} \, (x/e)^{x}$} \\
    &= v^2 \, \entropy{2}{p} - v \left( 2 \log_2{\left(\frac{v}{e} \right)} + \entropy{2}{p} \right) - \log_2{\left(2ve^2\right) }  
\end{align*}
Invoking Lemma~\ref{lemma:entropy_lower_bound}, one obtains: 
\begin{align}
    c_{f_\text{isom}}^\textit{both}(\bbB_{v,p}) 
    &\geq \min_{G_c} \frac{\entropy{2}{f_\text{isom}(G) | G_c} }{\log_2{s}} \notag \\
    &\geq v^2 \, \entropy{s}{p} - v \left( 2 \log_s{\left(\frac{v}{e} \right)} + \entropy{s}{p}\right) - \log_s{\left(2ve^2\right) } = \beta 
\end{align}
Then Lemma~\ref{lemma:c_single} gives:
\begin{align*}
	c_{f_\text{isom}}^{\textit{one}}(\bbB_{v,p})  \log_2{s}
	&\geq c_{f_\text{isom}}^\textit{both}(\bbB_{v,p})  \log_2{s} - \max_{G_b,G_c} \log_s{\left(|\{ f_\text{isom}(G_a,G_b,G_c) \, : \, G_a \in \cX_a \} |\right)} \log_2{s} \\
	&= c_{f_\text{isom}}^\textit{both}(\bbB_{v,p})  \log_2{s}  - \log_2{\left( \frac{|\cX_a|}{v!} \right)} \\	
	&= c_{f_\text{isom}}^\textit{both}(\bbB_{v,p})  \log_2{s} - \frac{v(v-1)}{2}  + \log_2{\left(v!\right)} \\	
	&\geq v(v-1) \left(\entropy{2}{p}-\frac{1}{2}\right) - \left( v \log_2{\left(\frac{v}{e} \right)}  + \frac{1}{2}\log_2{\left( ve^2 \right) } \right) - 1 \\
	&= v^2 \entropy{2}{p} - \frac{v}{2} \left( 2 \log_2{\left(\frac{v}{e} \right)} + \entropy{2}{p}\right)  - \frac{1}{2}\log_2{\left( 2ve^2 \right) } - \frac{v^2 - v + v \entropy{2}{p} + 1}{2} \\
	&= \frac{\beta \log_2{s}  - v^2 + v(1-\entropy{2}{p}) - 1}{2}
\end{align*}
implying  
$
c_{f_\text{isom}}^{\textit{one}}(\bbB_{v,p}) \geq \frac{ \beta}{2} - \frac{v^2 - v(1-\entropy{2}{p}) + 1}{2 \log_2{s}}. 
$
 
\subsection{Proof of Theorem~\ref{theorem:tree_isomorphism}}

According to \citet{otter1948number}, the number of unlabeled trees on $v$ nodes grows like 
$$
	t(v) \sim c \, \alpha^v \, v^{-5/2},
$$ where the values $c$ and $\alpha$ known to be approximately 0.5349496 and 2.9557652 (sequence A051491 in the OEIS). 
Moreover, it was shown in the proof of Theorem~\ref{theorem:graph_isomorphism_wc}, the number of monochromatic rectangles is at least $(t(v)+1) \, t(v)/2$.

Corollary~\ref{cor:max_rectangle_bound} then implies 
\begin{align*}
c_{f_\text{isom}}^\textit{both}
    &\geq \log_s{ \left( \frac{(t(v)+1) \, t(v)}{2} \right) } \\
	&\geq \log_s{ \left( \frac{t(v)^2}{2} \right)} \\
    &\sim 2 \log_s{ \left( \alpha^v \, v^{-5/2} \right) } - \log_s{(c^2/2)} 
	\sim 2v \log_s{\alpha} - 5 \log_s{v} + \log_s{7} = \beta
\end{align*} 
Further, from Lemma~\ref{lemma:c_single} one can derive:
\begin{align*}
	c_{f_\text{isom}}^\textit{one} 
	&\geq c_{f_\text{isom}}^\textit{both} - \max_{G_b,G_c} \log_s{\left(|\{ f(G_a,G_b,G_c) \, : \, G_a \in \cX_a \} |\right)} \\
	&= c_{f_\text{isom}}^\textit{both} - \log_s{t(v)} \\
	&\sim \log_s{ \left( \alpha^v \, v^{-5/2} \right) } - \log_s{(c/2)}	
 	\sim v \log_s{\alpha} - \frac{5}{2} \log_s{v} + \frac{1}{2}\log_s{14} 
\end{align*}
implying
$
c_{f_\text{isom}}^{\textit{one}} \geq \frac{ \beta + \log_{s}{2}}{2}. 
$

Let me now consider the case that $G$ is sampled uniformly at random from the set of all trees in $\cX$. It is a consequence of Lemma~\ref{lemma:entropy_lower_bound} that when the graph $(G_a,G_b,G_c) \sim \bbG$ (conditioned on $G_c$) is sampled uniformly at random from a collection of isomorphism classes, the expected communication complexity is at least  
$$ 
c_{f_\text{isom}}^\textit{both}(\mathbb{T}_{v}) \geq \min_{G_c} \log_s { | \{ f(G_a,G_b,G_c) \ : \ G \in \cX \text{ s.t. } G_c \} |}. 
$$ 
This can be seed to be identical to the worst-case bound encountered above. The derivation thus can be carried out analogously (and the same holds for $c_{f_\text{isom}}^\textit{one}(\mathbb{T}_{v})$ by Lemma~\ref{lemma:c_single}).

Finally, the upper bound $O(v)$ follows by the same argument as in the proof of Theorem~\ref{theorem:graph_isomorphism_wc}, where now the number of edges of each of $G_a$ and $G_b$ is $v-1$. 

\subsection{Proof of Lemma~\ref{lemma:mpnn-complexity}}

In general terms, the impossibility statement comes as a consequence of the definition of communication complexity: if the number of required exchanged symbols exceeds the symbols the learner can exchange (i.e., its communication capacity) then the latter will not be able to identify exactly $f_\text{isom}$. 

The specifics depend on the appropriate definition: 

Majority-voting necessitates $|\cal{M}_G|\geq \mu$, meaning that when $|\cal{M}_G| \geq \mu > n-2v$ at least one of the two parties should have gathered sufficient information to determine $f_{\text{isom}}(G)$ at the final layer. Therefore, $m$ should be ``one''. With consensus on the other hand, we have that $|\cal{M}_G| \geq n - \mu > n - v$ which implies that both parties need to know the class.

The worst-case communication complexity definition guarantees that there exists at least one input for which the required number of symbols is $c_{f_\text{isom}}^{(m)}$. Thus, since $\bbD$ is densely supported on $\cX$, the impossibility must occur with strictly positive probability. The impossibility also applies to any universe $\cX'$ that is a strict superset of $\cX$. This can be easily derived by conditioning on $\cX \subset \cX'$ (which can only decrease the communication complexity) and repeating the analysis identically.

The implications of the expected complexity bound, are two-fold: 

First, if $\net$ is adaptive, its capacity $c_\net$ is a random variable over the input distribution. The bound then asserts that $\E{c_\net} \geq c_{f_\text{isom}}^m(\bbD)$.

For networks of fixed size, one may derive a bound on the probability of error. Specifically, fix $\pi^*$ to be the protocol that achieves minimal expected length and let $\beta_m$ be an upper bound of $\pi^*$ length over all inputs. By Lemma~\ref{lemma:prob_lower_bound}, for any $\delta \in [0,1]$ one has
\begin{align*}
	\Prob{ \| \pi^*(G) \|_m >  \delta \, c_{f_\text{isom}}^m(\bbD)} 
	&\geq \frac{1-\delta}{(\beta_m/c_{f_\text{isom}}^m(\bbD)) - \delta}. 
\end{align*}
The above is a bound on the probability of error for a network that satisfies $c_\net \leq 2 \delta \, c_{f_\text{isom}}^m(\bbD)$.

One can also generalize the previous result to distributions $\bbD'$ defined on a strict superset $\cX'$ of $\cX$ that is (up to normalization) identical with $\bbD$ within $\cX$:
$$
\text{for all } G \in \cX: \quad \Prob{G \sim \bbD'} = c \, \Prob{G \sim \bbD} \quad \text{with} \quad c = \sum_{G \in X} \Prob{G \sim \bbD'}.  
$$
Then, the probability of error w.r.t. $\bbD'$ is at least $c \, \frac{1-\delta}{(\beta_m/c_{f_\text{isom}}^m(\bbD)) - \delta}$.

\section{Helpful lemmata}

\begin{lemma}
In the universe considered in Section~\ref{sec:communication_complexity}, the following hold for any $\bbD$:
\begin{align*}
	c_{f}^\textit{one} &\geq c_{f}^\textit{both} - \max_{G_b,G_c} \log_s{\left(|\{ f(G_a,G_b,G_c) \, : \, G_a \in \cX_a \} |\right)}\\
	c_{f}^\textit{one}(\bbD) &\geq c_{f}^\textit{both}(\bbD) - \max_{G_b,G_c} \log_s{\left(|\{ f(G_a,G_b,G_c) \, : \, G_a \in \cX_a \} |\right)}.
\end{align*}
\label{lemma:c_single}
\end{lemma}

\begin{proof}
Consider the setting of $c_{f}^{\textit{one}}$, where for a successful termination it suffices for one party to compute the output of $f$. Suppose w.l.o.g., that this party is Alice. In particular, Alice determines class $y = f(G_a,G_b,G_c)$ based on a protocol $\pi$ of minimal length. In this setting, Bob does not know $y$ but he is aware of $\cX_b^\ell$ (and $G_c$), where $\ell$ is the leaf of the protocol tree at input $(G_a,G_b,G_c)$. Therefore, both parties know that the class must belong to the set $\{ f(G_a,G_b,G_c) \, : \, G_b \in \cX_b^{\ell} \text{ and } G_a \in \cX_a \}$. It is a consequence that there exists a protocol $\pi'$ of length 
$$
	\|\pi'(G_a,G_b,G_c)\|_\textit{both} \leq \|\pi(G_a,G_b,G_c)\|_\textit{one} + \log_s{ |\{ f(G_a,G_b,G_c) \, : \, G_b \in \cX_b^{\ell} \text{ and } G_a \in \cX_a \}| }
$$
that results in both parties knowing $y$. The protocol $\pi'$ entails first simulating $\pi$ and then Alice sending to Bob the index of $y$ in the set of feasible classes. 
Moreover, since $f$ corresponds to the graph isomorphism problem, for Alice to know $y$, she must also know the isomorphism class of Bob. Therefore, the feasible set of classes contains only the feasible subgraph isomorphism classes of $G_a$, which are at most
\begin{align*}
	|\{ f(G_a,G_b,G_c) \, : \, G_b \in \cX_b^{\ell} \text{ and } G_a \in \cX_a \}| 
	&\leq \max_{G_b,G_c} |\{ f(G_a,G_b,G_c) \, : \, G_a \in \cX_a \} |
\end{align*}
The claimed inequalities then
follow by the optimality of the protocol $\pi$ and since the same construction can be repeated for every input.
\end{proof}

\begin{lemma}
Let $X$ be a categorical random variable with sample space $\mathcal{X}$. For any partitioning $\cX = \cA_1, \cdots, \cA_k$ we have that 
$$
    \entropy{s}{X} \geq \sum_{i = 1}^k \Prob{ X \in \cA_i } \log_s{ \left( \frac{1}{\Prob{ X \in \cA_i}} \right)}  
$$
\label{lemma:entropy_partitioning_ineq}
\end{lemma}
\begin{proof}
The proof is elementary. It relies on the inequality $\Prob{X = x} \leq \Prob{ X \in \cA_i}$ that holds for all $x \in \cA_i$:
\begin{align*}
    \entropy{2}{X}  
    &=  \sum_{ i= 1}^k \Prob{X \in \cA_i} \sum_{x \in \cA_i}   \frac{\Prob{X = x}}{\Prob{X \in \cA_i}} \log_s{ \left( \frac{1}{\Prob{ X =x}} \right)} \\
    &\geq \sum_{ i= 1}^k \Prob{X \in \cA_i} \min_{x \in \cA_i}  \log_s{ \left( \frac{1}{\Prob{ X =x}} \right)} \\
    &= \sum_{ i= 1}^k \Prob{X \in \cA_i} \log_s{ \left( \frac{1}{\max_{x \in \cA_i}  \Prob{ X =x}} \right)} \\
    &\geq \sum_{ i= 1}^k \Prob{X \in \cA_i} \log_s{ \left( \frac{1}{\sum_{x \in \cA_i}  \Prob{ X =x}} \right)}
    = \sum_{i = 1}^k \Prob{ X \in \cA_i } \log_s{ \left( \frac{1}{\Prob{ X \in \cA_i}} \right)},
 \end{align*}
as claimed.
\end{proof}

\begin{lemma}
For any random variable $X \leq \beta$ and $\delta\in [0,1]$ we have $\Prob{ X > \delta\, \E{X}} \geq \frac{1-\delta}{r - \delta}$, where $r = \beta/\E{X}$.
\label{lemma:prob_lower_bound}
\end{lemma}
\begin{proof}
For any $t \leq \beta$,
\begin{align}
	\E{X} = \sum_{x \leq t} \Prob{X} x + \sum_{x > t} \Prob{X} x  \leq \Prob{X \leq t } t + \Prob{X > t}\beta = (1-\Prob{X > t}) t + \Prob{X > t}\beta\notag 
\end{align}
or, equivalently,
$
	\Prob{X > t} \geq (\E{X} - t)/(\beta - t).   
$
The final inequality is obtained by setting $t = \delta \E{X}$.
\end{proof}

\end{document}